\documentclass{article} 
\usepackage{iclr2027_conference,times}

\usepackage{amsmath,amsfonts,bm}

\def\eqref#1{equation~\ref{#1}}

\def\1{\bm{1}}

\DeclareMathAlphabet{\mathsfit}{\encodingdefault}{\sfdefault}{m}{sl}
\SetMathAlphabet{\mathsfit}{bold}{\encodingdefault}{\sfdefault}{bx}{n}

\usepackage{hyperref}       
\usepackage{url}            
\usepackage{booktabs}       
\usepackage{amsfonts}       
\usepackage{nicefrac}       
\usepackage{microtype}      
\usepackage{xcolor}         
\usepackage{amsmath}
\usepackage{graphicx}
\usepackage{subcaption}
\usepackage{natbib}
\usepackage{multirow}
\usepackage{booktabs}
\usepackage{wrapfig}
\usepackage{amssymb}
\usepackage{amsfonts}
\usepackage{float}
\usepackage{makecell}
\usepackage{amsthm}

\newtheorem{theorem}{Theorem}
\newtheorem{lemma}[theorem]{Lemma}

\newtheorem{proposition}[theorem]{Proposition}

\usepackage[most]{tcolorbox}
\usepackage{amssymb}

\tcbset{aibox/.style={
    breakable,
    break at=\maxdimen,
    height fixed for=first and middle,
}}
\newcommand{\modelname}{KGPFN}
\newtcolorbox{AIbox}[2][]{aibox,title={#2},#1}
\usepackage{hyperref}
\usepackage{url}

\title{ KGPFN: Unlocking the Potential of Knowledge Graph Foundation Model via In-Context Learning}

\author{%
\normalfont
\makebox[\dimexpr\textwidth-2\tabcolsep\relax][c]{%
\begin{tabular}{@{}c@{}}
\textbf{Yisen Gao$^{1}$}, \textbf{Jiaxin Bai$^{2}$}, \textbf{Haoyu Huang$^{1}$}, \textbf{Zhongwei Xie$^{1}$}, \\
\textbf{Yufei Li$^{1}$}, \textbf{Hong Ting Tsang$^{1}$}, \textbf{Sirui Han$^{1}$}\thanks{Sirui Han is the corresponding author.} \\
\textbf{Yangqiu Song$^{1}$} \\[3pt]
Department of Computer Science and Engineering, HKUST, Hong Kong, China$^{1}$ \\
Department of Computer Science, HKBU, Hong Kong, China$^{2}$ \\[3pt]
{\texttt\small \{ygaodi\}@cse.ust.hk}
\end{tabular}}%
}

\iclrfinalcopy 
\begin{document}

\maketitle
\lhead{Preprint.}

\begin{abstract}
Knowledge graph (KG) foundation models aim to generalize to graphs with unseen entities and relations by learning transferable relational structure. Most existing methods, however, focus on relation-level universality, leaving in-context learning, the other pillar of foundation models, largely unexplored for KG reasoning. Context in KGs is structured and heterogeneous: accurate prediction requires conditioning both on the local neighborhood of the query entities and on global context that summarizes how the query relation behaves across many instances.
We propose KGPFN, a KG foundation model built on a Prior-Data Fitted Network (PFN) that combines transferable relational representations with inference-time in-context learning over structured context. KGPFN learns relation representations by message passing on relation graphs and extracts multi-scale local context from the intermediate head representations of a multi-layer NBFNet. It then builds relation-specific global context from positive and negative examples of the query relation, together with their local structural representations, and aggregates this context with feature-level and sample-level attention. Through multi-graph pretraining, KGPFN learns to combine structural representations with labeled contextual evidence without inference-time parameter updates. On 57 knowledge graphs, KGPFN achieves the best average MRR both without and with fine-tuning, and context sensitivity analyses highlight the value of negative context examples. Our code is available at https://github.com/HKUST-KnowComp/KGPFN.
\end{abstract}
\section{Introduction}

Foundation models have become a dominant paradigm, generalizing through large-scale pretraining and adapting to new tasks through context at inference time. Large language models~\citep{gpt4,deepseek,llama} show that a single pretrained model can perform diverse tasks specified by the prompt and a few in-context demonstrations~\citep{manyshot}.
Foundation models for structured data increasingly follow the Prior-Data Fitted Network (PFN)~\citep{pfn} paradigm, in which a model pretrained on synthetic tasks infers a new task from a small labeled context set at inference time. Examples include tabular models~\citep{tabpfnv1,tabpfnv2.5,tabicl,tabiclv2,limix}, time-series models~\citep{timepfn}, and extensions to relational data~\citep{relationalpfn}.
These models share two properties: \emph{broad transferability} across tasks and domains with little or no retraining, and \emph{strong in-context learning} from labeled examples provided at inference time.

Knowledge graphs (KGs) represent relational data in applications such as question answering~\citep{bai2023complex,unifying}, recommendation~\citep{kgforrec}, and drug discovery~\citep{ctrlhgen}, and underpin neural graph databases that manage and query graphs with learned representations~\citep{ngdb1,ngdb2,ngdb3,ngdb-zoo,ngdb-towards}.
Traditional KG embeddings~\citep{query2particles,query2box,transe} learn graph-specific entity representations and must be retrained for each new graph. KG foundation models~\citep{ultra,motif,kgicl,trix} instead learn transferable relational knowledge that generalizes across KGs with different entity and relation vocabularies, enabling zero-shot or few-shot adaptation to unseen graphs.

Existing KG foundation models have advanced transferable relational reasoning, and KGICL~\citep{kgicl} further incorporates in-context examples in the form of prompt graphs. Yet transferable structural patterns alone may not capture how a relation is instantiated in a particular graph, which calls for conditioning explicitly on relation-specific positive and negative examples at inference time.
Given the success of PFN-style structured foundation models~\citep{tabpfnv1,tabpfnv2.5,limix}, extending this paradigm to KGs is a natural direction, but not a straightforward one. Unlike text or tables, KG context is structured and heterogeneous: rather than a single surrounding window, it combines (i) \emph{local context}, the query-specific neighborhood in which reasoning takes place, and (ii) \emph{global context}, the relation-specific evidence accumulated across many instances.

\begin{figure}[t] 
  \centering
  \includegraphics[width=0.95\linewidth]{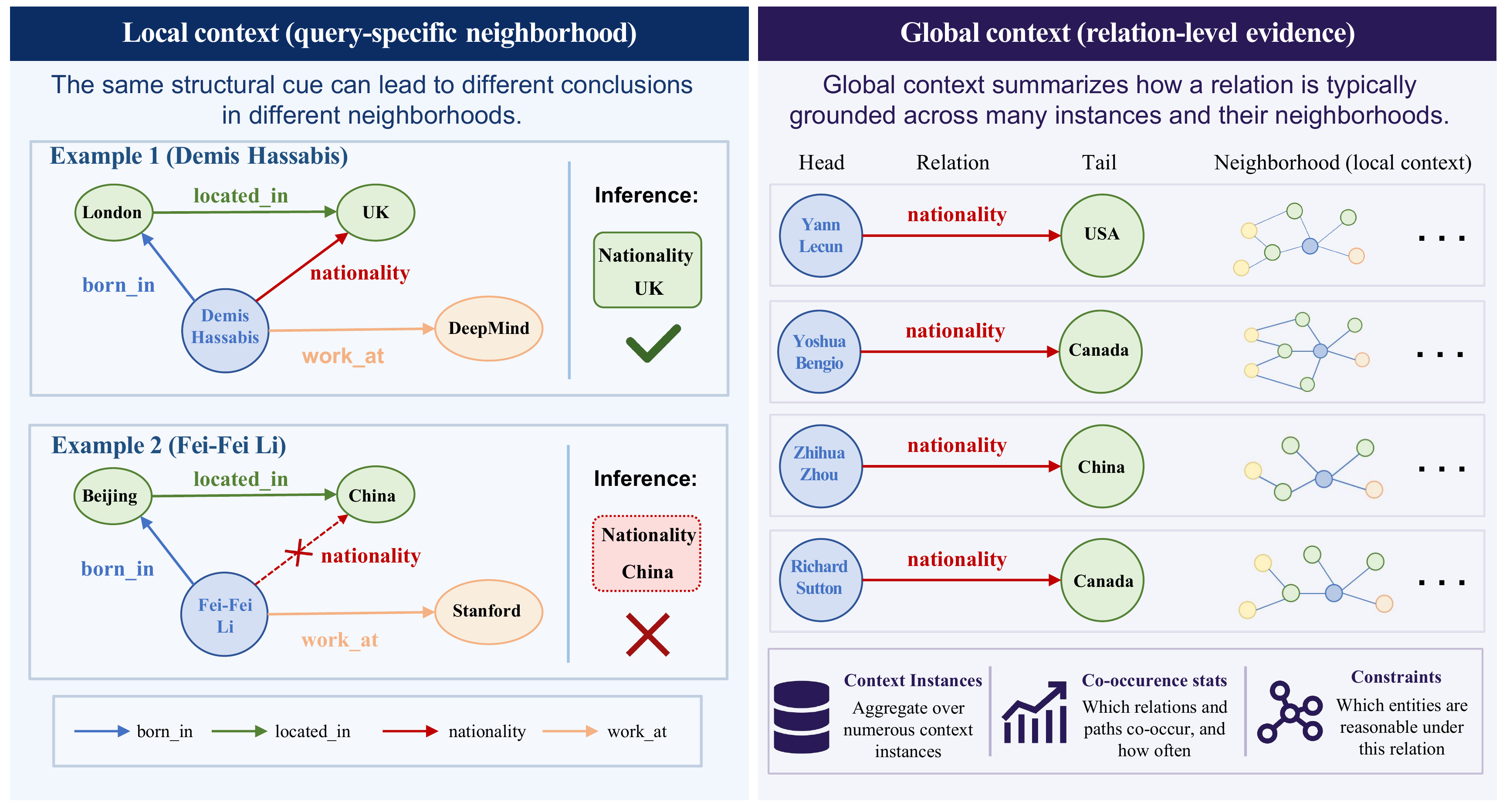} 
  \caption{(a) \textbf{Local context.} The same structural cue may support different conclusions depending on the query-specific neighborhood.
(b) \textbf{Global context.} Global context summarizes how a relation is typically instantiated across many examples and their neighborhoods.}
  \label{fig:intro}
\end{figure}

\textit{Local context is the subgraph around the query entities.} The same relation can call for different predictions depending on this neighborhood: a reusable structural cue may be reliable in one subgraph and misleading in another. As shown in Fig.~\ref{fig:intro} (left), the pattern $(\texttt{person},\texttt{born\_in},\texttt{city})\wedge(\texttt{city},\texttt{located\_in},\texttt{nation})$ often indicates $(\texttt{person},\texttt{nationality},\texttt{nation})$. For Demis Hassabis, $(\texttt{Demis Hassabis},\texttt{born\_in},\texttt{London})$ and $(\texttt{London},\texttt{located\_in},\texttt{UK})$ correctly suggest $(\texttt{Demis Hassabis},\texttt{nationality},\texttt{UK})$. For Fei-Fei Li, the analogous path through $\texttt{Beijing}$ points to $\texttt{China}$, yet nationality is not determined by birthplace alone: other edges in her neighborhood, such as $(\texttt{Fei\text{-}Fei Li},\texttt{works\_at},\texttt{Stanford})$, provide complementary evidence. A model must therefore condition on the local neighborhood to decide when a generic pattern applies and when it should be overridden.

\textit{Global context captures how a relation behaves across many instances rather than within a single neighborhood.} As shown in Fig.~\ref{fig:intro} (right), it can be formed from representative instances of the query relation together with their local neighborhoods. This collection offers a distributional view of how the relation is grounded in graph structure: which auxiliary relations and multi-hop paths tend to co-occur with it, how often they do so, and how they constrain plausible heads and tails. Relational universality and context are thus complementary: the former captures reusable relational structure, while local and global context indicate how that structure applies to a particular graph and query. This motivates combining transferable structural representations with relation-specific in-context examples.

    

Building on this analysis, we propose \textbf{KGPFN}, which unlocks the potential of \textbf{K}nowledge \textbf{G}raph foundation models through in-context learning with a \textbf{P}rior-data \textbf{F}itted \textbf{N}etwork. KGPFN couples transferable structural representations with inference-time use of context. It first runs message passing over the relation graph to obtain relation representations that capture transferable, relation-level structure. For local context, a multi-layer NBFNet computes neighborhood-aware representations of the query entities. For global context, we retrieve positive and negative instances of the query relation together with their local neighborhoods, and process them with a PFN that combines feature-level and sample-level attention. Pretrained on multiple KGs, KGPFN learns to jointly exploit structural regularities and contextual evidence, enabling zero-shot transfer and in-context adaptation to unseen graphs. Our main contributions are as follows:

\begin{itemize}
  \item \textbf{Local and global context for KG foundation models.} 
  We formulate KG context as local context, the query-specific neighborhood, and global context, labeled positive and negative instances of the query relation, moving beyond reliance on transferable relational patterns alone.

  \item \textbf{In-context learning over KG context with Prior-data Fitted Networks.} 
  \modelname{} performs in-context learning over this structured context: feature-level attention models structural signals within each instance, and sample-level attention captures regularities across instances, enabling adaptation to unseen graphs without parameter updates.

  \item \textbf{Evaluation on 57 knowledge graphs.}
  Across transductive, inductive, and fully inductive settings, \modelname{} achieves the best average MRR both without and with fine-tuning. Sensitivity analyses clarify the roles of local-context depth and of positive and negative global examples.
\end{itemize}

\section{Related Work}

\textbf{Knowledge graph foundation models.}
KG foundation models aim to capture relational invariances that transfer across heterogeneous KGs, enabling fully inductive, zero-shot completion on unseen entities and relations.
ULTRA~\citep{ultra} builds a relation graph from four fundamental interaction patterns and learns transferable, query-conditioned relation representations.
TRIX~\citep{trix} increases expressiveness by encoding which entities participate in relation interactions, and jointly updates entity and relation embeddings, which also enables efficient relation prediction.
MOTIF~\citep{motif} shows that many KG foundation models rely on binary motifs, and that higher-order motifs on relational hypergraphs yield greater expressiveness.
KGICL~\citep{kgicl} frames KG reasoning as in-context learning, using a few query-centered prompt graphs (e.g., 5-shot) and a unified tokenizer to align entities and relations across KGs.
FLOCK~\citep{flock} instead captures local context by learning from random walks.

\textbf{In-context learning for structured data via Prior-data Fitted Networks.}
Prior-data Fitted Networks (PFNs)~\citep{pfn} cast in-context learning as approximate Bayesian inference: a model pretrained on synthetic tasks drawn from a prior predicts query labels from labeled context examples in a single forward pass, without task-specific optimization.
TabPFN~\citep{tabpfnv1} applies this idea to tables by treating rows as context examples; TabPFN-v2.5~\citep{tabpfnv2.5}, TabICL~\citep{tabicl}, and Limix~\citep{limix} further improve robustness and scalability.
PFN-based models have also been extended to time series~\citep{timepfn}, relational databases~\citep{relationalpfn}, and graphs~\citep{graphpfn,tfmlinker}.
Compared with KGICL's few-shot prompt graphs, KGPFN conditions on a larger labeled set of positive and negative examples of the query relation. Compared with PFN-based graph models, it targets multi-relational KGs whose entity and relation vocabularies change across graphs.

\section{Preliminary}

\subsection{Task Definition}

A knowledge graph (KG) is a directed multi-relational graph 
\(\mathcal{G}=(\mathcal{E},\mathcal{R},\mathcal{T})\) with entity set \(\mathcal{E}\), relation set \(\mathcal{R}\), and observed triple set \(\mathcal{T}\subseteq \mathcal{E}\times\mathcal{R}\times\mathcal{E}\). Each triple \((h,r,t)\in\mathcal{T}\) is a directed edge from head entity \(h\) to tail entity \(t\) under relation \(r\). 
We study link prediction: given an incomplete query $(h,r,?)$, predict the missing tail entity $t$. A model assigns a plausibility score $s(h,r,t)$ to each candidate $t\in\mathcal{E}$ and ranks candidates accordingly. We distinguish the visible support triples used for graph encoding from the held-out triples to be predicted. Our benchmarks cover three split settings: \emph{transductive}, with shared entity and relation vocabularies; \emph{inductive}, with new test entities; and \emph{fully inductive}, with new test entities and relations.

\subsection{Prior-data Fitted Networks}

\textbf{Posterior predictive distribution.}
In the Bayesian view of supervised learning, a prior $p(\phi)$ is placed over a hypothesis space $\Phi$ of mappings from inputs $x$ to labels $y$. Each hypothesis $\phi\in\Phi$ induces a data-generating distribution from which a dataset $D=\{(x_i,y_i)\}_{i=1}^{n}$ is sampled. Given a context set $D$ and a test input $x$, the posterior predictive distribution marginalizes over hypotheses:
\begin{equation}
p(y \mid x, D)\ \propto\ \int_{\Phi} p(y \mid x,\phi)\, p(D\mid \phi)\, p(\phi)\, d\phi,
\label{eq:ppd}
\end{equation}
where $p(D\mid \phi)$ is the likelihood of $D$ under $\phi$.

\textbf{Prior-data fitted networks.}
PFNs cast in-context learning as amortized Bayesian inference: a single network is trained to approximate the posterior predictive distribution given labeled context examples. Given a support set $\mathcal{C}=\{(x_c^{(i)},y_c^{(i)})\}_{i=1}^{n}$ and a query $x_q$, a PFN predicts
\begin{equation}
f_\theta(x_q,\mathcal{C})\ \approx\ p(y_q \mid x_q, \mathcal{C}).
\label{eq:pfn_approx}
\end{equation}
The prior over tasks is encoded implicitly in the pretrained weights, and conditioning on the observed context $\mathcal{C}$ yields an approximate posterior predictive distribution in a single forward pass, without parameter updates.

\section{Method}
\label{sec:method}

\begin{figure*}[t]
  \centering
  \includegraphics[width=\linewidth]{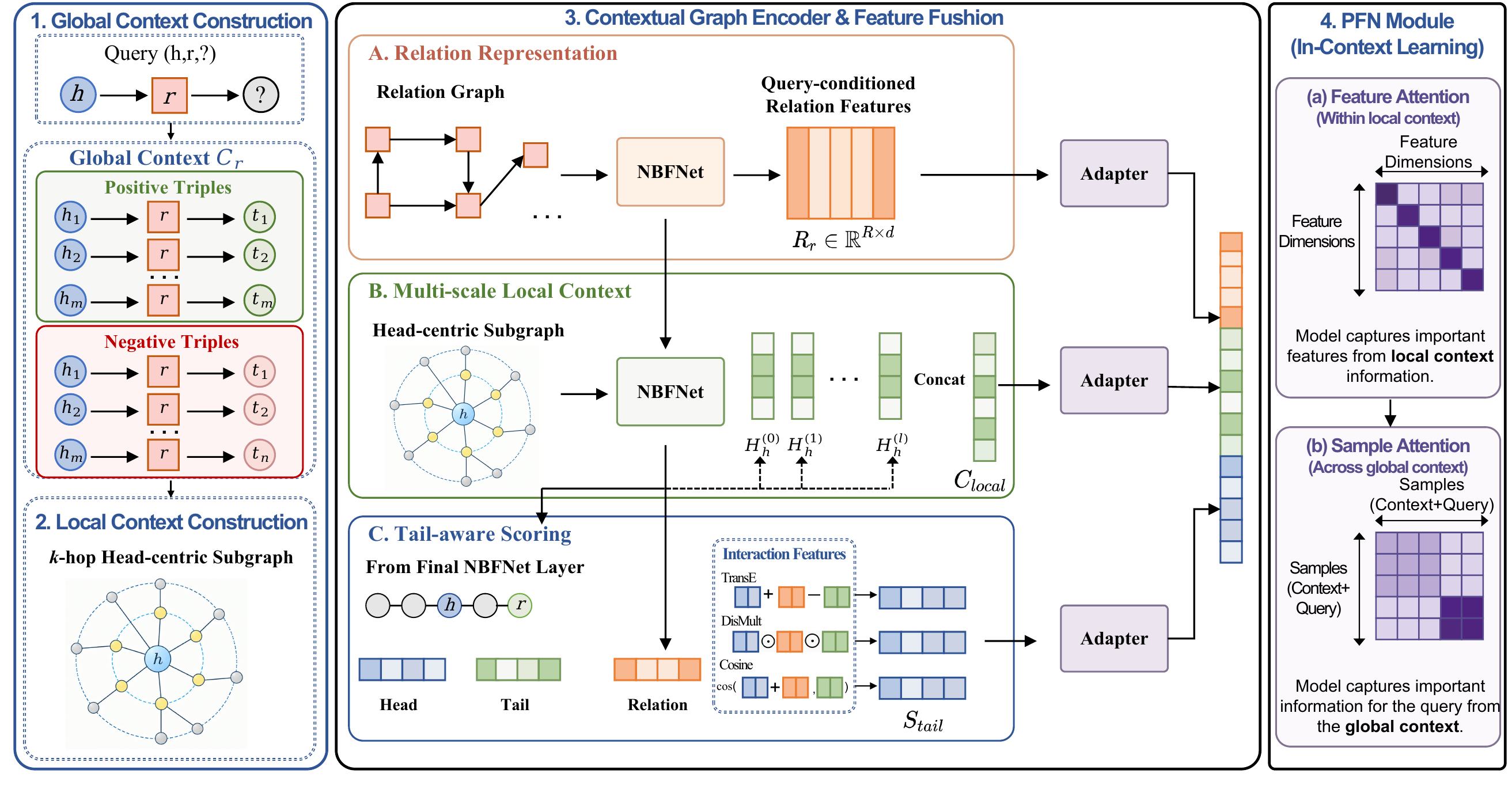} 
  \caption{The framework of KGPFN. }
  \label{fig:kgpfn_framework}
\end{figure*}

This section presents KGPFN, a knowledge graph foundation model with in-context learning capabilities; its framework is shown in Fig.~\ref{fig:kgpfn_framework}. Sections~\ref{sec:method:global} and~\ref{sec:method:local} construct the global and local context, Section~\ref{sec:method:encoder} encodes them, and Section~\ref{sec:method:pfn} describes the PFN, which combines feature-level and sample-level attention.

\subsection{Global Context Construction}
\label{sec:method:global}

To show the model how the query relation is typically grounded, we construct a relation-specific global context. Let \(\mathcal{G}_{\mathrm{obs}}=(\mathcal{E},\mathcal{R},\mathcal{T}_{\mathrm{obs}})\) denote the graph of visible support edges under the evaluation protocol; context construction and graph encoding use only this graph, so held-out query edges are excluded. For a query \((h,r,?)\), we collect observed triples of the same relation, \(\{(h_i,r,t_i)\}_{i=1}^{m^+}\), from \(\mathcal{T}_{\mathrm{obs}}\), and generate negative triples of that relation, \(\{(h_j,r,t_j)\}_{j=1}^{m^-}\), by negative sampling (Appendix~\ref{app:training}). Together they form the global context \(\mathcal{C}_r=\{((h_k,r,t_k),y_k)\}_{k=1}^{m}\), where \(m=m^+ + m^-\) and \(y_k\in\{0,1\}\) labels each example as positive or negative.

\subsection{Local Context Construction}
\label{sec:method:local}

Whereas global context provides relation-level evidence, local context captures the neighborhood structure around each query or context triple. Extracting an enclosing subgraph around both endpoints of every candidate triple would be expensive, so we use a head-centric local context shared by all candidate tails of a given \((h,r)\).
Conceptually, this context is the \(k\)-hop neighborhood of \(h\) in \(\mathcal{G}_{\mathrm{obs}}\). We obtain its representation through head-conditioned NBFNet message passing (Section~\ref{sec:method:encoder}): because each layer aggregates evidence from one additional hop, the head states of the first \(k\) layers provide multi-scale local-context features. The same graph-wide forward pass also yields final-layer representations of all candidate tails, so local-context extraction and all-entity scoring share their graph computation.

\subsection{Contextual Graph Encoder}
\label{sec:method:encoder}

We next encode both kinds of context into vector representations.

\textbf{Relation representations.}
Following~\citet{ultra} and~\citet{trix}, we build a relation graph \(\mathcal{G}_r=(\mathcal{R},\mathcal{R}_{\mathrm{fund}},\mathcal{E}_r)\) from \(\mathcal{G}_{\mathrm{obs}}\). Each node is a relation type in \(\mathcal{R}\), and edges in \(\mathcal{E}_r\) connect relations that share endpoints, typed by four fundamental interactions \(\mathcal{R}_{\mathrm{fund}}=\{\mathrm{t2h},\mathrm{h2h},\mathrm{h2t},\mathrm{t2t}\}\). Equivalently, the relation graph is an adjacency tensor \(\mathbf{A}_r\in\mathbb{R}^{|\mathcal{R}|\times|\mathcal{R}|\times 4}\).

Given a query relation \(r\), we initialize its node with an indicator (\(\mathbf{h}^{(0)}_{r\mid r}=\mathbf{1}^d\) and \(\mathbf{h}^{(0)}_{v\mid r}=\mathbf{0}\) for \(v\neq r\)) and run message passing on \(\mathcal{G}_r\), where each layer aggregates messages over the four interaction types \(s\in\mathcal{R}_{\mathrm{fund}}\):
\begin{equation}
\mathbf{h}_{v\mid r}^{(\ell+1)} \;=\; \mathrm{Upd}\!\left(
\mathbf{h}_{v\mid r}^{(\ell)},
\;\;
\mathrm{Agg}\Big(\big\{\mathrm{Msg}_s(\mathbf{h}_{u\mid r}^{(\ell)}) :
u\in \mathcal{N}_s(v),\ s\in \mathcal{R}_{\mathrm{fund}}\big\}\Big)
\right).
\label{eq:rel_mp}
\end{equation}
After \(L_r\) layers, we obtain the query-conditioned relation embeddings \(\mathbf{R}_{r}=\big[\mathbf{h}_{v\mid r}^{(L_r)}\big]_{v\in\mathcal{R}}\in\mathbb{R}^{|\mathcal{R}|\times d}\).

\textbf{Multi-scale local context representation.}
We run an \(L\)-layer NBFNet on \(\mathcal{G}_{\mathrm{obs}}\) conditioned on \((h,r)\), initializing the head node with the query relation embedding and all other nodes with zeros: \(\mathbf{H}_h^{(0)}=\mathbf{r}\) and \(\mathbf{H}_v^{(0)}=\mathbf{0}\) for \(v\neq h\). Writing \(\mathbf{H}_v^{(\ell)}\) for the state of node \(v\) at layer \(\ell\) (with the conditioning on \((h,r)\) left implicit), each layer propagates relation-aware messages along visible typed edges:
\begin{equation}
\mathbf{H}_v^{(\ell+1)} \;=\; \mathrm{Upd}\!\left(
\mathbf{H}_v^{(\ell)},
\;\;
\mathrm{Agg}\Big(\big\{\mathrm{Msg}(\mathbf{H}_u^{(\ell)}, r_{u\!\to\! v}, \mathbf{r}) :
(u,r_{u\!\to\! v},v)\in\mathcal{T}_{\mathrm{obs}}\big\}\Big)
\right).
\label{eq:nbfnet_update_head}
\end{equation}
Given the relation representations, each layer extends the receptive field by one hop (Lemma~\ref{lemma:locality_app}). We therefore use the head states as local-context summaries, \(\mathbf{c}_{\mathrm{loc}}^{(\ell)}(h,r)=\mathbf{H}_{h}^{(\ell)}\), and stack them as \(\mathbf{C}_{\mathrm{loc}}(h,r)=[\mathbf{c}_{\mathrm{loc}}^{(1)};\cdots;\mathbf{c}_{\mathrm{loc}}^{(k)}]\) with \(1\leq k\leq L\). 

\textbf{Tail-aware scoring representation.}
Let \(\mathbf{z}_v^{(L)}=\mathbf{H}_v^{(L)}\). To discriminate among candidate tails, we combine the final-layer tail embedding \(\mathbf{z}_{t}^{(L)}\), the head embedding \(\mathbf{z}_{h}^{(L)}\), and the query relation embedding \(\mathbf{r}\) into three interaction features:
\begin{equation}
\begin{aligned}
\phi_{\mathrm{transe}}(h,r,t)&=\mathbf{z}_{h}^{(L)}+\mathbf{r}-{\mathbf{z}}_{t}^{(L)},\\
\phi_{\mathrm{distmult}}(h,r,t)&=\mathbf{z}_{h}^{(L)}\odot\mathbf{r}\odot{\mathbf{z}}_{t}^{(L)},\\
\phi_{\mathrm{cos}}(h,r,t)&=\cos\!\big(\mathbf{z}_{h}^{(L)}+\mathbf{r},\,{\mathbf{z}}_{t}^{(L)}\big),
\end{aligned}
\end{equation}
which are concatenated as \(S_{\mathrm{tail}}(h,r,t)=[\phi_{\mathrm{transe}}(h,r,t),\phi_{\mathrm{distmult}}(h,r,t),\phi_{\mathrm{cos}}(h,r,t)]\).
For the query and each context example \(((h_k,r,t_k),y_k)\in\mathcal{C}_r\), three adapters project the local context, the triple embedding, and \(S_{\mathrm{tail}}\) into a shared space, and their outputs are concatenated:
\begin{equation}
\mathbf{x}_k
=
\Big[
\mathrm{Adapter}\!\big(\mathbf{C}_{\mathrm{loc}}(h_k,r)\big)
\;;\;
\mathrm{Adapter}\!\big([\mathbf{z}_{h_k}^{(L)};\mathbf{r};\mathbf{z}_{t_k}^{(L)}]\big)
\;;\;
\mathrm{Adapter}\!\big(S_{\mathrm{tail}}(h_k,r,t_k)\big)
\Big],
\label{eq:context_adapter_fuse}
\end{equation}
where each \(\mathrm{Adapter}(\cdot)\) is a lightweight MLP.

\subsection{PFN Architecture}
\label{sec:method:pfn}

Given the fused representations in Eq.~\ref{eq:context_adapter_fuse}, the PFN (Eq.~\ref{eq:pfn_approx}) scores a query \(\mathbf{x}_q\) conditioned on the labeled context \(\mathcal{C}_r=\{(\mathbf{x}_i,y_i)\}_{i=1}^{m}\), where \(y_i\in\{0,1\}\) indicates whether context triple \(i\) is positive.
Feature-level self-attention, with RoPE~\citep{rope} positional encoding, lets the features of each example interact. Sample-level attention then aggregates the context: context examples attend to one another through self-attention, and each candidate query retrieves relation-specific evidence through cross-attention to \(\mathcal{C}_r\). Keys and values come only from context examples, while residual connections preserve each candidate's own representation. The contextualized query representation is mapped to the score \(s(h,r,t)\). Appendix~\ref{app:efficiency} details the formulation, inference complexity, and caching strategy.

\textbf{Data prior construction.}
Tabular PFNs~\citep{tabpfnv1,tabpfnv2.5,limix} often synthesize tasks from random generative processes. We instead build an empirical task prior by sampling support--query episodes from the training KGs, which preserves their multi-relational structure. Each episode consists of a support graph \(\mathcal{G}_{\mathrm{s}}\) and a query graph \(\mathcal{G}_{\mathrm{q}}\) drawn from the remaining training triples whose endpoints appear in \(\mathcal{G}_{\mathrm{s}}\). The support graph plays the role of \(\mathcal{G}_{\mathrm{obs}}\) for context construction and graph encoding, and the held-out query triples serve as prediction targets.
To sample \(\mathcal{G}_{\mathrm{s}}\), we draw a base keep ratio \(\rho\) shared across relations, add a small relation-specific perturbation \(\delta\), and retain each relation's edges at the resulting rate. Training over many such episodes yields a predictor that transfers across graphs by conditioning on labeled graph context.

\textbf{Training objective}.
We train KGPFN with two complementary losses under negative sampling. For a batch of \(B\) positive triples with scores \(s_i^{+}\), each paired with \(N\) negatives with scores \(s_{i,j}^{-}\), a binary cross-entropy (BCE) loss separates positives from negatives:
\begin{equation}
\mathcal{L}_{\mathrm{BCE}}
=
-\frac{1}{B}\sum_{i=1}^{B}
\left[
\log \sigma\!\left(s_i^{+}\right)
+
\sum_{j=1}^{N} w_{i,j}\,\log\!\left(1-\sigma\!\left(s_{i,j}^{-}\right)\right)
\right],
\label{eq:loss_bce}
\end{equation}
where $\sigma(\cdot)$ is the sigmoid function and $w_{i,j}=\frac{\exp(s_{i,j}^{-}/\tau)}{\sum_{k=1}^{N}\exp(s_{i,k}^{-}/\tau)}$ is a self-adversarial weight with temperature $\tau$ that emphasizes hard negatives.
A softmax cross-entropy loss further encourages each positive to rank above its sampled negatives:
\begin{equation}
\mathcal{L}_{\mathrm{CE}}
=
-\frac{1}{B}\sum_{i=1}^{B}
\log
\frac{\exp\!\left(s_i^{+}\right)}
{\exp\!\left(s_i^{+}\right)+\sum_{j=1}^{N}\exp\!\left(s_{i,j}^{-}\right)}.
\label{eq:loss_ce}
\end{equation}
The overall objective is $\mathcal{L}=\mathcal{L}_{\mathrm{BCE}}+ \mathcal{L}_{\mathrm{CE}}$.

\paragraph{An idealized view of context-dependent scoring.}
To illustrate how labeled context can modulate structural scoring without parameter updates, we analyze an idealized bilinear surrogate of the context interaction.

\begin{proposition}[Context-dependent feature reweighting in a bilinear surrogate]
\label{prop:context_matching}
Fix a feature map \(\mathbf{c}:\mathcal{X}\to\mathbb{R}^{n}\) on triples, a bias \(b\in\mathbb{R}\), weights \(\mathbf{w}_0\in\mathbb{R}^{n}\), and a matrix \(\mathbf{A}\in\mathbb{R}^{n\times n}\). For a query \(q\in\mathcal{X}\) and a labeled support set \(\mathcal{S}_r=\{(s_i,y_i)\}_{i=1}^{K}\), with \(K\geq1\), \(y_i\in\{0,1\}\), and \(\widetilde y_i=2y_i-1\), define
\begin{equation}
    s_{\mathrm{bil}}(q\mid\mathcal{S}_r)
    :=b+\mathbf{w}_0^\top\mathbf{c}(q)
    +\frac{1}{K}\sum_{i=1}^{K}\widetilde y_i\,\mathbf{c}(q)^\top\mathbf{A}\mathbf{c}(s_i).
    \label{eq:bilinear_surrogate}
\end{equation}
Writing \(\boldsymbol{\mu}_{\mathcal{S}_r}:=K^{-1}\sum_{i=1}^{K}\widetilde y_i\mathbf{c}(s_i)\), the score satisfies
\begin{equation}
    s_{\mathrm{bil}}(q\mid\mathcal{S}_r)
    =b+\big(\mathbf{w}_0+\mathbf{A}\boldsymbol{\mu}_{\mathcal{S}_r}\big)^\top\mathbf{c}(q).
    \label{eq:context_effective_weights}
\end{equation}
\end{proposition}

Proposition~\ref{prop:context_matching} shows that labeled support examples shift the effective structural weights by \(\mathbf{A}\boldsymbol{\mu}_{\mathcal{S}_r}\), while the feature map and model parameters stay fixed. A diagonal \(\mathbf{A}\) reweights features coordinate-wise, a general \(\mathbf{A}\) also couples different query and support features, and \(\mathbf{A}=\mathbf{0}\) recovers a context-independent structural scorer. Appendix~\ref{app:theory_detailed} gives the proof and the exact condition under which changing the support set changes a query score.

\section{Experiments}

\subsection{Experimental Setup}
\label{sec:expsetup}
\textbf{Datasets and metrics.}
Following prior work~\citep{ultra,trix}, we evaluate on 57 knowledge graphs from diverse domains, grouped by their benchmark splits: \emph{transductive} (16 datasets), where training and test graphs share entity and relation vocabularies; \emph{inductive} (18 datasets), with new test entities; and \emph{fully inductive} (23 datasets), with new test entities and relations. These categories describe each benchmark's own training/test split; whether pretrained parameters are further fine-tuned is determined by the evaluation protocols below.
We use filtered ranking metrics: for each test triple $(h,r,t)$, the tail is ranked against all candidate entities after filtering out other true triples. For KGPFN, we report mean MRR and Hits@10 over five inference runs with different context-sampling seeds and a fixed checkpoint per dataset and setting; per-dataset standard deviations are given in Appendix~\ref{app:moreresults}. Benchmark averages weight each graph equally. Dataset statistics are listed in Appendix~\ref{app:dataset}.

\textbf{Baselines.}
We compare KGPFN with five KG foundation models: \textbf{ULTRA}~\citep{ultra}, which pretrains transferable relation representations on relation graphs for inductive reasoning across graphs; \textbf{KGICL}~\citep{kgicl}, which performs in-context reasoning over retrieved prompt graphs; \textbf{TRIX}~\citep{trix}, a more expressive fully inductive model for zero-shot transfer; \textbf{MOTIF}~\citep{motif}, which captures higher-order relational motifs beyond pairwise relation interactions; and \textbf{FLOCK}~\citep{flock}, which learns from random walks.

\textbf{Implementation details.} 
Following~\citet{ultra}, the relation encoder and the local-context encoder are each 6-layer NBFNets with hidden dimension 64. 
The PFN Transformer is initialized from pretrained TabICL~\citep{tabiclv2} weights, with its feature preprocessing module removed.
Local context is read from the head-node states of the first three NBFNet layers, and the sixth layer provides representations of all candidate tails in the visible support graph.
The global context contains 20 positive and 80 negative triples. We train on 8 NVIDIA A800 80GB GPUs and run inference on a single A800 GPU.
Further training details are given in Appendix~\ref{app:training}.

\subsection{Evaluation without and with Fine-tuning}
Following prior KG foundation models~\citep{ultra,trix,motif}, we pretrain KGPFN on FB15k-237~\citep{fb15k}, WN18RR~\citep{wn18rr}, and CoDEx-Medium~\citep{codex}, and evaluate under two protocols. Without fine-tuning, the pretrained model retrieves positive and negative context examples from each target graph's visible support edges and predicts with fixed parameters. This measures zero-shot transfer on the 54 graphs outside the pretraining set and held-out link prediction on the three pretraining graphs. With fine-tuning, the model is further trained on the target dataset's training split. Because KGICL uses a different pretraining setup, we include it only in the fine-tuned comparison.

Table~\ref{tab:avg_57kgs} summarizes the average results. Without fine-tuning, KGPFN achieves the best MRR in all three settings and overall, and ties with FLOCK for the best overall Hits@10. Its results in the two inductive settings show that the pretrained in-context predictor transfers to graphs with unseen entities and relations. With fine-tuning, KGPFN achieves the best average MRR and Hits@10 in every setting, raising overall MRR from 0.396 to 0.434. Fine-tuning brings large gains on several graphs whose domains differ markedly from the pretraining graphs, such as ConceptNet100k (0.116 to 0.291), AristoV4 (0.205 to 0.342), and Hetionet (0.276 to 0.405). Per-dataset results are reported in Appendix~\ref{app:moreresults}, and inference efficiency is compared in Appendix~\ref{app:efficiency_benchmark}.

\begin{table*}[t]
  \centering
  \small
  \setlength{\tabcolsep}{4pt}
  \caption{Average MRR and Hits@10 over 57 KGs. The 16-graph transductive and 57-graph overall averages include the three pretraining graphs (FB15k-237, WN18RR, and CoDEx-Medium); the ``Pretraining'' columns report this subset separately. W/o FT: without fine-tuning. (\textbf{Bold}: best; \underline{Underline}: runner-up.)}
  \label{tab:avg_57kgs}
  \begin{tabular}{l l cc cc cc cc cc}
    \toprule
    \multirow{2}{*}{Setting} & \multirow{2}{*}{Model}
   & \multicolumn{2}{c}{\makecell{Inductive $e,r$\\(23 graphs)}}
& \multicolumn{2}{c}{\makecell{Inductive $e$\\(18 graphs)}}
& \multicolumn{2}{c}{\makecell{Transductive\\(16 graphs)}}
& \multicolumn{2}{c}{\makecell{Total Avg\\(57 graphs)}}
& \multicolumn{2}{c}{\makecell{Pretraining\\(3 graphs)}} \\
    \cmidrule(lr){3-4}\cmidrule(lr){5-6}\cmidrule(lr){7-8}\cmidrule(lr){9-10}\cmidrule(lr){11-12}
      & & MRR & H@10 & MRR & H@10 & MRR & H@10 & MRR & H@10 & MRR & H@10 \\
    \midrule

    \multirow{5}{*}{W/o FT}
      & ULTRA & 0.345 & 0.513 & 0.431 & 0.566 & 0.329 & 0.479 & 0.366 & 0.518 & -- & -- \\
    & MOTIF & 0.349 & 0.525 & 0.436 & 0.577 & 0.332 & 0.476 & 0.372 & 0.528 & -- & -- \\
      & TRIX  & 0.368 & 0.540 & 0.455 & 0.592 & \underline{0.353} & 0.512 & 0.391 & 0.551 & -- & -- \\
         & FLOCK  & \underline{0.369} & \underline{0.554} & \underline{0.456} & \textbf{0.604} & 0.352 & \underline{0.518} & \underline{0.392} & \textbf{0.560} & --  & --\\
          & KGPFN  & \textbf{0.373} & \textbf{0.558} & \textbf{0.460} & \underline{0.597} & \textbf{0.355} & \textbf{0.521} & \textbf{0.396} & \textbf{0.560} & --  & --\\
    \midrule
    \multirow{6}{*}{Fine-tuned}
        & ULTRA & 0.397 & 0.556 & 0.442 & 0.582 & 0.384 & 0.548 & 0.408 & 0.562 & 0.407 & \underline{0.568} \\
      & KGICL & 0.410 & 0.560 & 0.467 & 0.604 & \underline{0.397} & 0.554 & 0.419 & 0.576 & -- & -- \\
    & MOTIF & 0.401 & 0.558 & 0.455 & 0.594 & 0.394 & 0.552 & 0.417 & 0.569 & \underline{0.415} & 0.565 \\
      & TRIX  & 0.401 & 0.556 & 0.459 & 0.594 & 0.395 & \underline{0.559} & 0.418 & 0.569 & \underline{0.415} & 0.563 \\
       & FLOCK  & \underline{0.417} & \underline{0.576} & \underline{0.473} & \underline{0.619} & 0.389 & 0.547 & \underline{0.427} & \underline{0.582} & \underline{0.415} & 0.561 \\
     & KGPFN & \textbf{0.421} & \textbf{0.590} & \textbf{0.482} & \textbf{0.623} & \textbf{0.398} & \textbf{0.567} & \textbf{0.434} & \textbf{0.594} & \textbf{0.416} & \textbf{0.580} \\
    \bottomrule
  \end{tabular}
\end{table*}

\begin{figure}[t]
    \centering
    \includegraphics[width=\linewidth]{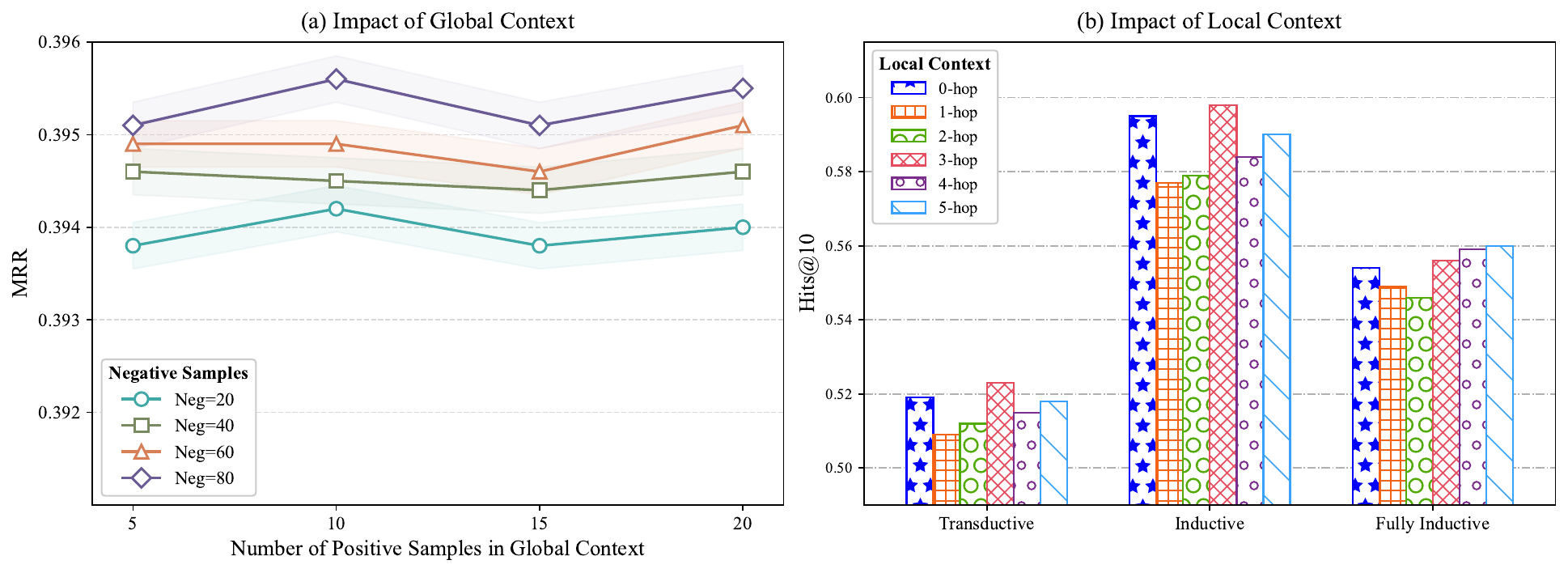}
    \caption{Context sensitivity. (a) Average MRR over 57 KGs for different numbers of positive and negative global-context examples. (b) Hits@10 for different local-context depths.}
    \label{fig:context}
\end{figure}


\subsection{Context Sensitivity Analysis}
We study how global-context size and local-context depth affect performance, retraining the model for each configuration. Since the PFN predicts from a nonempty labeled support set, we vary the numbers of positive and negative examples rather than removing global context entirely.

\textbf{Global context.} We vary the number of positive examples over \{5, 10, 15, 20\} and the number of negative examples over \{20, 40, 60, 80\}, evaluate all 16 combinations, and report the average MRR over the 57 graphs (Fig.~\ref{fig:context}, left).
Within this range, performance is relatively insensitive to the number of positives, whereas adding negatives consistently improves MRR. This suggests that negative examples provide useful evidence for discriminating among candidates. A likely reason is that about 80\% of negatives are drawn from the head's two-hop neighborhood (Appendix~\ref{app:training}), so they act as hard, structurally similar contrasts to the true tails.

\textbf{Local context.} We vary the local-context depth from 0 to 5 hops (Fig.~\ref{fig:context}, right); depth 0 removes the local-context input \(\mathbf{C}_{\mathrm{loc}}\) and its adapter branch from Eq.~\ref{eq:context_adapter_fuse}. Three-hop context gives the best Hits@10 in the transductive and inductive settings and remains competitive in the fully inductive setting. With only 1 or 2 hops, performance can even fall below the variant without local context, suggesting that very shallow summaries are insufficient for disambiguation. Going beyond 3 hops brings no consistent further gain.

\subsection{Ablation Study}
To assess each tail-aware scoring feature, we remove the TransE, DistMult, and cosine features one at a time (Table~\ref{tab:ablation_score}).
Removing any single feature lowers MRR in all three settings, indicating that the features are complementary. Removing the TransE-style feature causes the largest drop, from 0.396 to 0.389 in average MRR.

\subsection{Inference Efficiency}
\label{sec:efficiency}
We measure per-query latency and peak GPU memory for all-entity ranking on FB15k-237 (Table~\ref{tab:inference_efficiency}, Appendix~\ref{app:efficiency_benchmark}). KGPFN takes 62.1\,ms per query with 2.78\,GB of peak memory. Its latency matches that of KGICL (61.8\,ms), which also builds query-specific context, while using about $9\times$ less memory (24.9\,GB), and it is about $8\times$ faster than FLOCK (491\,ms, 50.9\,GB). Direct structural scorers remain faster (0.95--18.7\,ms for ULTRA, TRIX, and MOTIF) because they skip context retrieval and encoding. KGPFN keeps this overhead moderate by sharing one graph-encoder pass across all candidate tails and caching the encoded support set across candidate chunks.

\subsection{Visualization}
\begin{wrapfigure}{r}{0.42\textwidth}
  \centering
  \vspace{-0.5\baselineskip} 
  \includegraphics[width=0.40\textwidth]{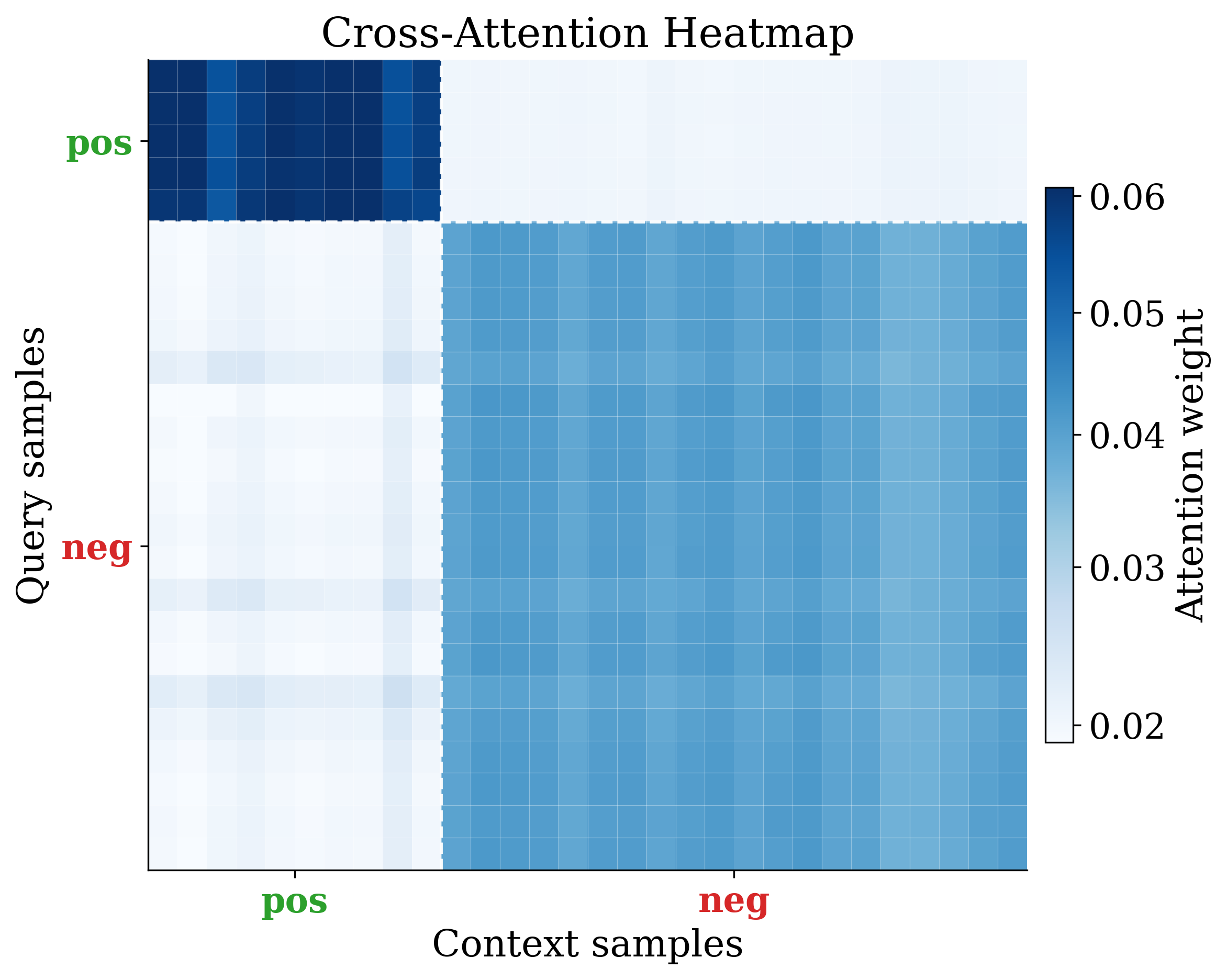}
  \caption{Last-layer sample attention from query candidates to context examples.}
  \label{fig:visualization}
  \vspace{-0.5\baselineskip} 
\end{wrapfigure}
To illustrate how the model uses labeled global context, we present a case study on CoDEx-Large~\citep{codex} with the query (\texttt{Imperial}, \texttt{Record label}$^{-1}$, ?), where $^{-1}$ denotes the inverse relation. The candidate set contains 5 correct and 20 incorrect tails, and the global context contains 5 positive and 20 negative examples.
Fig.~\ref{fig:visualization} shows the cross-attention weights from the last sample-attention layer, i.e., how each query candidate attends to the context examples. Since candidate rows contribute no keys or values, these weights show directly which labeled examples each candidate draws on. Correct candidates concentrate their attention on positive context examples, whereas incorrect candidates spread their attention more diffusely over negative examples. The case study thus illustrates that attention to labeled global context is candidate-specific.

\begin{table}[t]
  \centering
  \small
  \caption{Ablation of tail-aware scoring features (MRR). (\textbf{Bold}: best.)}
  \label{tab:ablation_score}
  \begin{tabular}{l cccc}
    \toprule
    Model
    & \makecell{Inductive $e,r$\\(23 graphs)}
    & \makecell{Inductive $e$\\(18 graphs)}
    & \makecell{Transductive\\(16 graphs)}
    & \makecell{Average\\(57 graphs)} \\
    \midrule
    KGPFN & \textbf{0.373} & \textbf{0.460} & \textbf{0.355} & \textbf{0.396} \\
    KGPFN (w/o TransE score) & 0.365 & 0.454 & 0.351 & 0.389 \\
    KGPFN (w/o DistMult score) & 0.368 & 0.458 & 0.353 & 0.392 \\
    KGPFN (w/o cosine score) & 0.370 & 0.456 & 0.352 & 0.392 \\
    \bottomrule
  \end{tabular}
\end{table}

\section{Conclusion}
Knowledge graph foundation models have mainly pursued transferable relational patterns, while in-context learning has received less attention. Effective KG reasoning draws on two kinds of context: the local neighborhood of the query, and a global view of how the query relation behaves across many instances.
\modelname{} combines both. It learns relation representations by message passing on the relation graph, reads multi-scale local context from intermediate NBFNet head states, and aggregates positive and negative examples of the query relation with a Prior-Data Fitted Network that uses feature-level and sample-level attention. Labeled contextual evidence thus informs prediction at inference time without parameter updates.
On 57 knowledge graphs spanning transductive, inductive, and fully inductive settings, \modelname{} performs strongly both without and with fine-tuning, and the context analyses highlight the value of negative examples and three-hop local context. Retrieving and encoding context makes inference slower than direct structural scorers such as ULTRA (Section~\ref{sec:efficiency}); reducing this cost is a natural direction for future work. Other promising extensions include selecting context examples adaptively for each query instead of sampling them, and applying the same in-context formulation to complex logical queries and to temporal or multimodal knowledge graphs.

\bibliography{iclr2027_conference}

@article{gpt4,
  author       = {OpenAI},
  title        = {{GPT-4} Technical Report},
  journal      = {CoRR},
  volume       = {abs/2303.08774},
  year         = {2023},
  url          = {https://doi.org/10.48550/arXiv.2303.08774},
  doi          = {10.48550/ARXIV.2303.08774},
  eprinttype   = {arXiv},
  eprint       = {2303.08774},
  bibsource    = {dblp computer science bibliography, https://dblp.org}
}

@article{deepseek,
  title={Deepseek-v3 technical report},
  author={Liu, Aixin and Feng, Bei and Xue, Bing and Wang, Bingxuan and Wu, Bochao and Lu, Chengda and Zhao, Chenggang and Deng, Chengqi and Zhang, Chenyu and Ruan, Chong and others},
  journal={arXiv preprint arXiv:2412.19437},
  year={2024}
}

@article{llama,
  title={Llama: Open and efficient foundation language models},
  author={Touvron, Hugo and Lavril, Thibaut and Izacard, Gautier and Martinet, Xavier and Lachaux, Marie-Anne and Lacroix, Timoth{\'e}e and Rozi{\`e}re, Baptiste and Goyal, Naman and Hambro, Eric and Azhar, Faisal and others},
  journal={arXiv preprint arXiv:2302.13971},
  year={2023}
}

@article{tabpfnv1,
  title={Tabpfn: A transformer that solves small tabular classification problems in a second},
  author={Hollmann, Noah and M{\"u}ller, Samuel and Eggensperger, Katharina and Hutter, Frank},
  journal={arXiv preprint arXiv:2207.01848},
  year={2022}
}

@article{tabpfnv2.5,
  title={Tabpfn-2.5: Advancing the state of the art in tabular foundation models},
  author={Grinsztajn, L{\'e}o and Fl{\"o}ge, Klemens and Key, Oscar and Birkel, Felix and Jund, Philipp and Roof, Brendan and J{\"a}ger, Benjamin and Safaric, Dominik and Alessi, Simone and Hayler, Adrian and others},
  journal={arXiv preprint arXiv:2511.08667},
  year={2025}
}

@inproceedings{pfn,
  title={Statistical foundations of prior-data fitted networks},
  author={Nagler, Thomas},
  booktitle={International Conference on Machine Learning},
  pages={25660--25676},
  year={2023},
  organization={PMLR}
}

@article{tabicl,
  title={Tabicl: A tabular foundation model for in-context learning on large data},
  author={Qu, Jingang and Holzm{\~A}{\v{z}}ller, David and Varoquaux, Ga{\~A}l and Morvan, Marine Le},
  journal={arXiv preprint arXiv:2502.05564},
  year={2025}
}

@article{tabiclv2,
  title={TabICLv2: A better, faster, scalable, and open tabular foundation model},
  author={Qu, Jingang and Holzm{\"u}ller, David and Varoquaux, Ga{\"e}l and Morvan, Marine Le},
  journal={arXiv preprint arXiv:2602.11139},
  year={2026}
}

@article{limix,
  title={Limix: Unleashing structured-data modeling capability for generalist intelligence},
  author={Zhang, Xingxuan and Ren, Gang and Yu, Han and Yuan, Hao and Wang, Hui and Li, Jiansheng and Wu, Jiayun and Mo, Lang and Mao, Li and Hao, Mingchao and others},
  journal={arXiv preprint arXiv:2509.03505},
  year={2025}
}

@inproceedings{timepfn,
  title={Timepfn: Effective multivariate time series forecasting with synthetic data},
  author={Taga, Ege Onur and Ildiz, Muhammed Emrullah and Oymak, Samet},
  booktitle={Proceedings of the AAAI conference on artificial intelligence},
  volume={39},
  number={19},
  pages={20761--20769},
  year={2025}
}

@article{relationalpfn,
  title={Relational In-Context Learning via Synthetic Pre-training with Structural Prior},
  author={Wang, Yanbo and You, Jiaxuan and Shi, Chuan and Zhang, Muhan},
  journal={arXiv preprint arXiv:2603.03805},
  year={2026}
}

@article{ngdb1,
  title={Top ten challenges towards agentic neural graph databases},
  author={Bai, Jiaxin and Wang, Zihao and Zhou, Yukun and Yin, Hang and Fei, Weizhi and Hu, Qi and Deng, Zheye and Cheng, Jiayang and Zheng, Tianshi and Tsang, Hong Ting and others},
  journal={arXiv preprint arXiv:2501.14224},
  year={2025}
}

@inproceedings{ngdb2,
  title={Neural graph databases},
  author={Besta, Maciej and Iff, Patrick and Scheidl, Florian and Osawa, Kazuki and Dryden, Nikoli and Podstawski, Michal and Chen, Tiancheng and Hoefler, Torsten},
  booktitle={Learning on Graphs Conference},
  pages={31--1},
  year={2022},
  organization={PMLR}
}

@article{ngdb3,
  title={Neural graph reasoning: Complex logical query answering meets graph databases},
  author={Ren, Hongyu and Galkin, Mikhail and Cochez, Michael and Zhu, Zhaocheng and Leskovec, Jure},
  journal={arXiv preprint arXiv:2303.14617},
  year={2023}
}

@article{ngdb-zoo,
  title={NGDB-Zoo: Towards Efficient and Scalable Neural Graph Databases Training},
  author={Xie, Zhongwei and Bai, Jiaxin and Liu, Shujie and Huang, Haoyu and Li, Yufei and Gao, Yisen and Tsang, Hong Ting and Song, Yangqiu},
  journal={arXiv preprint arXiv:2602.21597},
  year={2026}
}

@article{ngdb-towards,
  title={Towards Neural Graph Data Management},
  author={Li, Yufei and Gao, Yisen and Bai, Jiaxin and Xiong, Jiaxuan and Huang, Haoyu and Xie, Zhongwei and Tsang, Hong Ting and Song, Yangqiu},
  journal={arXiv preprint arXiv:2603.05529},
  year={2026}
}

@article{ultra,
  title={Towards foundation models for knowledge graph reasoning},
  author={Galkin, Mikhail and Yuan, Xinyu and Mostafa, Hesham and Tang, Jian and Zhu, Zhaocheng},
  journal={arXiv preprint arXiv:2310.04562},
  year={2023}
}

@article{kgicl,
  title={A prompt-based knowledge graph foundation model for universal in-context reasoning},
  author={Cui, Yuanning and Sun, Zequn and Hu, Wei},
  journal={Advances in Neural Information Processing Systems},
  volume={37},
  pages={7095--7124},
  year={2024}
}

@article{motif,
  title={How Expressive are Knowledge Graph Foundation Models?},
  author={Huang, Xingyue and Barcel{\'o}, Pablo and Bronstein, Michael M and Ceylan, Ismail Ilkan and Galkin, Mikhail and Reutter, Juan L and Orth, Miguel Romero},
  journal={arXiv preprint arXiv:2502.13339},
  year={2025}
}

@article{trix,
  title={TRIX: A more expressive model for zero-shot domain transfer in knowledge graphs},
  author={Zhang, Yucheng and Bevilacqua, Beatrice and Galkin, Mikhail and Ribeiro, Bruno},
  journal={arXiv preprint arXiv:2502.19512},
  year={2025}
}

@article{ctrlhgen,
  title={Controllable logical hypothesis generation for abductive reasoning in knowledge graphs},
  author={Gao, Yisen and Bai, Jiaxin and Zheng, Tianshi and Sun, Qingyun and Zhang, Ziwei and Fu, Xingcheng and Li, Jianxin and Song, Yangqiu},
  journal={arXiv preprint arXiv:2505.20948},
  year={2025}
}

@inproceedings{unifying,
  title={Unifying Deductive and Abductive Reasoning in Knowledge Graphs with Masked Diffusion Model},
  author={Gao, Yisen and Bai, Jiaxin and Huang, Yi and Fu, Xingcheng and Sun, Qingyun and Song, Yangqiu},
  booktitle={Proceedings of the ACM Web Conference 2026},
  pages={3600--3611},
  year={2026}
}

@article{bai2023complex,
  title={Complex query answering on eventuality knowledge graph with implicit logical constraints},
  author={Bai, Jiaxin and Liu, Xin and Wang, Weiqi and Luo, Chen and Song, Yangqiu},
  journal={Advances in Neural Information Processing Systems},
  volume={36},
  pages={30534--30553},
  year={2023}
}

@inproceedings{kgforrec,
  title={Unifying knowledge graph learning and recommendation: Towards a better understanding of user preferences},
  author={Cao, Yixin and Wang, Xiang and He, Xiangnan and Hu, Zikun and Chua, Tat-Seng},
  booktitle={The world wide web conference},
  pages={151--161},
  year={2019}
}

@inproceedings{query2particles,
  title={Query2particles: Knowledge graph reasoning with particle embeddings},
  author={Bai, Jiaxin and Wang, Zihao and Zhang, Hongming and Song, Yangqiu},
  booktitle={Findings of the Association for Computational Linguistics: NAACL 2022},
  pages={2703--2714},
  year={2022}
}

@article{query2box,
  title={Query2box: Reasoning over knowledge graphs in vector space using box embeddings},
  author={Ren, Hongyu and Hu, Weihua and Leskovec, Jure},
  journal={arXiv preprint arXiv:2002.05969},
  year={2020}
}

@inproceedings{fb15k,
  title={Observed versus latent features for knowledge base and text inference},
  author={Toutanova, Kristina and Chen, Danqi},
  booktitle={Proceedings of the 3rd workshop on continuous vector space models and their compositionality},
  pages={57--66},
  year={2015}
}

@inproceedings{wn18rr,
  title={Convolutional 2d knowledge graph embeddings},
  author={Dettmers, Tim and Minervini, Pasquale and Stenetorp, Pontus and Riedel, Sebastian},
  booktitle={Proceedings of the AAAI conference on artificial intelligence},
  volume={32},
  number={1},
  year={2018}
}

@inproceedings{codex,
  title={Codex: A comprehensive knowledge graph completion benchmark},
  author={Safavi, Tara and Koutra, Danai},
  booktitle={Proceedings of the 2020 Conference on Empirical Methods in Natural Language Processing (EMNLP)},
  pages={8328--8350},
  year={2020}
}

@inproceedings{yago,
  title={Yago: a core of semantic knowledge},
  author={Suchanek, Fabian M and Kasneci, Gjergji and Weikum, Gerhard},
  booktitle={Proceedings of the 16th international conference on World Wide Web},
  pages={697--706},
  year={2007}
}

@article{nell,
  title={Never-ending learning},
  author={Mitchell, Tom and Cohen, William and Hruschka, Estevam and Talukdar, Partha and Yang, Bishan and Betteridge, Justin and Carlson, Andrew and Dalvi, Bhavana and Gardner, Matt and Kisiel, Bryan and others},
  journal={Communications of the ACM},
  volume={61},
  number={5},
  pages={103--115},
  year={2018},
  publisher={ACM New York, NY, USA}
}

@inproceedings{conceptnet,
  title={Conceptnet 5.5: An open multilingual graph of general knowledge},
  author={Speer, Robyn and Chin, Joshua and Havasi, Catherine},
  booktitle={Proceedings of the AAAI conference on artificial intelligence},
  volume={31},
  number={1},
  year={2017}
}

@article{dbpedia,
  title={Dbpedia--a large-scale, multilingual knowledge base extracted from wikipedia},
  author={Lehmann, Jens and Isele, Robert and Jakob, Max and Jentzsch, Anja and Kontokostas, Dimitris and Mendes, Pablo N and Hellmann, Sebastian and Morsey, Mohamed and Van Kleef, Patrick and Auer, S{\"o}ren and others},
  journal={Semantic web},
  volume={6},
  number={2},
  pages={167--195},
  year={2015},
  publisher={SAGE Publications Sage UK: London, England}
}

@article{hetionet,
  title={Systematic integration of biomedical knowledge prioritizes drugs for repurposing},
  author={Himmelstein, Daniel Scott and Lizee, Antoine and Hessler, Christine and Brueggeman, Leo and Chen, Sabrina L and Hadley, Dexter and Green, Ari and Khankhanian, Pouya and Baranzini, Sergio E},
  journal={elife},
  volume={6},
  pages={e26726},
  year={2017},
  publisher={eLife Sciences Publications, Ltd}
}

@inproceedings{aristov4,
  title={Meta relational learning for few-shot link prediction in knowledge graphs},
  author={Chen, Mingyang and Zhang, Wen and Zhang, Wei and Chen, Qiang and Chen, Huajun},
  booktitle={Proceedings of the 2019 conference on empirical methods in natural language processing and the 9th international joint conference on natural language processing (EMNLP-IJCNLP)},
  pages={4217--4226},
  year={2019}
}

@inproceedings{grial,
  title={Inductive relation prediction by subgraph reasoning},
  author={Teru, Komal and Denis, Etienne and Hamilton, Will},
  booktitle={International conference on machine learning},
  pages={9448--9457},
  year={2020},
  organization={PMLR}
}

@article{galkin2022open,
  title={An open challenge for inductive link prediction on knowledge graphs},
  author={Galkin, Mikhail and Berrendorf, Max and Hoyt, Charles Tapley},
  journal={arXiv preprint arXiv:2203.01520},
  year={2022}
}

@article{indigo,
  title={Indigo: Gnn-based inductive knowledge graph completion using pair-wise encoding},
  author={Liu, Shuwen and Grau, Bernardo and Horrocks, Ian and Kostylev, Egor},
  journal={Advances in Neural Information Processing Systems},
  volume={34},
  pages={2034--2045},
  year={2021}
}

@article{indigo2,
  title={Knowledge transfer for out-of-knowledge-base entities: A graph neural network approach},
  author={Hamaguchi, Takuo and Oiwa, Hidekazu and Shimbo, Masashi and Matsumoto, Yuji},
  journal={arXiv preprint arXiv:1706.05674},
  year={2017}
}

@inproceedings{ingram,
  title={Ingram: Inductive knowledge graph embedding via relation graphs},
  author={Lee, Jaejun and Chung, Chanyoung and Whang, Joyce Jiyoung},
  booktitle={International conference on machine learning},
  pages={18796--18809},
  year={2023},
  organization={PMLR}
}

@article{zhou2023ood,
  title={An ood multi-task perspective for link prediction with new relation types and nodes},
  author={Zhou, Jincheng and Bevilacqua, Beatrice and Ribeiro, Bruno},
  journal={arXiv preprint arXiv:2307.06046},
  volume={23},
  year={2023}
}

@article{graphpfn,
  title={GraphPFN: A prior-data fitted graph foundation model},
  author={Eremeev, Dmitry and Platonov, Oleg and Bazhenov, Gleb and Babenko, Artem and Prokhorenkova, Liudmila},
  journal={arXiv preprint arXiv:2509.21489},
  year={2025}
}

@article{tfmlinker,
  title={TFMLinker: Universal Link Predictor by Graph In-Context Learning with Tabular Foundation Models},
  author={Liao, Tianyin and Hu, Chunyu and Sui, Yicheng and Zhang, Xingxuan and Cui, Peng and Li, Jianxin and Zhang, Ziwei},
  journal={arXiv preprint arXiv:2602.08592},
  year={2026}
}

@article{rope,
  title={Roformer: Enhanced transformer with rotary position embedding},
  author={Su, Jianlin and Ahmed, Murtadha and Lu, Yu and Pan, Shengfeng and Bo, Wen and Liu, Yunfeng},
  journal={Neurocomputing},
  volume={568},
  pages={127063},
  year={2024},
  publisher={Elsevier}
}

@article{manyshot,
  title={Many-shot in-context learning},
  author={Agarwal, Rishabh and Singh, Avi and Zhang, Lei and Bohnet, Bernd and Rosias, Luis and Chan, Stephanie and Zhang, Biao and Anand, Ankesh and Abbas, Zaheer and Nova, Azade and others},
  journal={Advances in Neural Information Processing Systems},
  volume={37},
  pages={76930--76966},
  year={2024}
}

@article{transe,
  title={Translating embeddings for modeling multi-relational data},
  author={Bordes, Antoine and Usunier, Nicolas and Garcia-Duran, Alberto and Weston, Jason and Yakhnenko, Oksana},
  journal={Advances in neural information processing systems},
  volume={26},
  year={2013}
}

@inproceedings{flock,
  title={Flock: A knowledge graph foundation model via learning on random walks},
  author={Kim, Jinwoo and Huang, Xingyue and Olejniczak, Krzysztof and Min, Kyungbin and Bronstein, Michael and Hong, Seunghoon and Ceylan, Ismail I},
  booktitle={International Conference on Learning Representations},
  volume={2026},
  pages={36664--36705},
  year={2026}
}
\bibliographystyle{iclr2027_conference}

\appendix

\section{Theoretical Analysis of an Idealized Context Interaction}
\label{app:theory_detailed}

This appendix establishes conditional locality of entity-level message passing and analyzes context-dependent feature reweighting in an idealized bilinear surrogate.

\subsection{Finite-hop Structural Locality}
Let \(\mathcal{G}=(\mathcal{V},\mathcal{E},\mathcal{R})\) be the graph on which entity-level message passing is performed. Write \(u\xrightarrow{\rho}v\) when an edge of type \(\rho\) carries a message from \(u\) to \(v\). Define \(\mathcal{N}_L(v)\) to contain all nodes that can reach \(v\) in at most \(L\) such message-passing steps, including \(v\), together with the induced typed edges.

\begin{lemma}[Finite-hop Locality of Conditional Message Passing]
\label{lemma:locality_app}
Fix the model parameters and all conditioning inputs \(\boldsymbol{\eta}\), including query-conditioned relation embeddings. Suppose the entity updates have the form
\begin{equation}
    \mathbf{h}^{(\ell+1)}_v
    =F_\ell\!\left(
    \mathbf{h}^{(\ell)}_v,\,
    \{\!\{(\rho,\mathbf{h}^{(\ell)}_u):u\xrightarrow{\rho}v\}\!\};
    \boldsymbol{\eta}\right),
    \label{eq:locality_update}
\end{equation}
where \(F_\ell\) is a deterministic function of the displayed arguments and the braces denote a multiset. Conditional on \(\boldsymbol{\eta}\), \(\mathbf{h}^{(L)}_v\) is determined by the initial node features and typed edges in \(\mathcal{N}_L(v)\).
\end{lemma}

\begin{proof}
We prove the claim by induction on the layer index. At layer \(0\), the state of \(v\) is its given initial feature, contained in \(\mathcal{N}_0(v)\). Suppose the claim holds at layer \(\ell\) for every node. Equation~\ref{eq:locality_update} uses only the current state of \(v\), the types of its incoming message edges, and the current states of the corresponding neighbors \(u\). By the induction hypothesis, these states are determined by \(\mathcal{N}_\ell(v)\) and the neighborhoods \(\mathcal{N}_\ell(u)\). Every node that can reach such a neighbor \(u\) in at most \(\ell\) steps can reach \(v\) in at most \(\ell+1\) steps. Thus all required nodes, typed edges, and initial features lie in \(\mathcal{N}_{\ell+1}(v)\). Since \(F_\ell\) and \(\boldsymbol{\eta}\) are fixed, they determine \(\mathbf{h}^{(\ell+1)}_v\), completing the induction.
\end{proof}

\paragraph{Scope of locality.}
The lemma applies to the local entity-update form in Eq.~\ref{eq:locality_update}, conditional on the initial features and relation representations. In KGPFN, the latter are computed from a graph-wide relation graph. Thus the intermediate head states combine local entity-level propagation with graph-wide relational conditioning.

\subsection{Bilinear Surrogate and Proof of Feature Reweighting}

For Proposition~\ref{prop:context_matching}, let \(\mathcal{X}\) be a domain of triples and fix a feature map \(\mathbf{c}:\mathcal{X}\to\mathbb{R}^{n}\). The observed encoding graph and relation-conditioning inputs are held fixed when comparing support sets. This isolates the effect of context aggregation on a fixed structural representation.

Fix \(b\in\mathbb{R}\), \(\mathbf{w}_0\in\mathbb{R}^{n}\), and \(\mathbf{A}\in\mathbb{R}^{n\times n}\). For a nonempty support set \(\mathcal{S}=\{(s_i,y_i)\}_{i=1}^{K}\), let \(\widetilde y_i=2y_i-1\) and
\begin{equation}
    \boldsymbol{\mu}_{\mathcal{S}}
    :=\frac{1}{K}\sum_{i=1}^{K}\widetilde y_i\mathbf{c}(s_i).
\end{equation}
The score \(s_{\mathrm{bil}}\) is defined by Eq.~\ref{eq:bilinear_surrogate}. No symmetry, positive-semidefiniteness, or diagonal constraint on \(\mathbf{A}\) is required.

\begin{proof}[Proof of Proposition~\ref{prop:context_matching}]
By linearity of matrix multiplication and finite summation,
\begin{align}
    \frac{1}{K}\sum_{i=1}^{K}\widetilde y_i\,
    \mathbf{c}(q)^\top\mathbf{A}\mathbf{c}(s_i)
    &=\mathbf{c}(q)^\top\mathbf{A}
    \left(\frac{1}{K}\sum_{i=1}^{K}\widetilde y_i\mathbf{c}(s_i)\right) \nonumber\\
    &=\mathbf{c}(q)^\top\mathbf{A}\boldsymbol{\mu}_{\mathcal{S}} \nonumber\\
    &=(\mathbf{A}\boldsymbol{\mu}_{\mathcal{S}})^\top\mathbf{c}(q).
\end{align}
The last equality holds because both sides are the same real scalar; it does not require \(\mathbf{A}\) to be symmetric. Adding \(b+\mathbf{w}_0^\top\mathbf{c}(q)\) yields Eq.~\ref{eq:context_effective_weights}.
\end{proof}

\paragraph{When does context change the score?}
For any two nonempty support sets \(\mathcal{S}_1,\mathcal{S}_2\), which may have different sizes, subtraction gives the exact identity
\begin{equation}
    s_{\mathrm{bil}}(q\mid\mathcal{S}_1)-s_{\mathrm{bil}}(q\mid\mathcal{S}_2)
    =\mathbf{c}(q)^\top\mathbf{A}
    \big(\boldsymbol{\mu}_{\mathcal{S}_1}-\boldsymbol{\mu}_{\mathcal{S}_2}\big).
    \label{eq:context_score_difference}
\end{equation}
Thus the query score changes if and only if the right-hand side is nonzero: the change in effective weights must have a nonzero projection onto the query feature vector.

\paragraph{Coordinate-wise and cross-feature effects.}
The effective coefficient of query feature \(c_k(q)\) is
\begin{equation}
    w_{\mathrm{eff},k}(\mathcal{S})
    =w_{0,k}+\sum_{j=1}^{n}A_{kj}
    \left(\frac{1}{K}\sum_{i=1}^{K}\widetilde y_i c_j(s_i)\right).
\end{equation}
For \(\mathbf{A}=\mathrm{diag}(\lambda_1,\ldots,\lambda_n)\), only the matching support coordinate contributes to each correction. A general matrix additionally permits cross-feature contributions. The statistic \(\boldsymbol{\mu}_{\mathcal{S}}\) is a signed empirical feature average that reflects both the support features and their positive/negative class proportions.

\paragraph{Relationship to projected interactions.}
A projected dot product is one instance of the bilinear interaction. If \(\mathbf{z}_x=\mathbf{W}_{\mathcal{M}}\mathbf{c}(x)\), where \(\mathbf{W}_{\mathcal{M}}\in\mathbb{R}^{d\times n}\), and \(\mathbf{W}_Q,\mathbf{W}_K\in\mathbb{R}^{a\times d}\), then
\begin{equation}
    (\mathbf{W}_Q\mathbf{z}_q)^\top(\mathbf{W}_K\mathbf{z}_s)
    =\mathbf{c}(q)^\top
    \underbrace{\mathbf{W}_{\mathcal{M}}^\top\mathbf{W}_Q^\top
    \mathbf{W}_K\mathbf{W}_{\mathcal{M}}}_{\mathbf{A}}
    \mathbf{c}(s).
\end{equation}
This identity motivates the surrogate through projected query--support interactions. Proposition~\ref{prop:context_matching} describes the resulting feature reweighting under label-weighted bilinear aggregation. KGPFN implements context interaction using the multi-layer attention architecture detailed below.

\section{Detailed Formulations of PFN Architecture}
\label{app:efficiency}
\label{app:pfn_architecture_details}

In our framework, the PFN acts as an in-context relational inference engine. Given an in-context support set \(\mathcal{C}_r = \{(\mathbf{x}_i, y_i)\}_{i=1}^{m}\) and a set of query candidate triples \(\{\mathbf{x}_{q, j}\}_{j=1}^{N}\) conditioned on relation \(r\), the PFN estimates the plausibility score \(s(h, r, t)\). Following the design principles of TabICL~\citep{tabicl,tabiclv2}, we decouple the processing into two orthogonal Transformer modules: \textbf{Feature-Level Attention} (row-wise feature interaction) and \textbf{Sample-Level Attention} (dataset-wise in-context learning).

\paragraph{Feature-Level Attention (Row-wise Feature Interaction).}
For each triple \(k\) (whether an in-context triple from \(\mathcal{C}_r\) or a query candidate), the multi-source fusion in Eq.~\ref{eq:context_adapter_fuse} produces a structured set of \(F\) feature tokens:
\begin{equation}
\mathbf{X}_k = [\mathbf{x}_{k, 1}, \mathbf{x}_{k, 2}, \dots, \mathbf{x}_{k, F}] \in \mathbb{R}^{F \times d_{\mathrm{feat}}},
\end{equation}
which encapsulates the multi-hop local context \(\mathbf{C}_{\mathrm{loc}}(h_k, r)\), the conditional entity-relation representations \([\mathbf{z}_{h_k}^{(L)}; \mathbf{r}; \mathbf{z}_{t_k}^{(L)}]\), and the structural plausibility features \(S_{\mathrm{tail}}(h_k, r, t_k)\).

Following TabICL~\citep{tabicl,tabiclv2}, we first project the feature channels into the internal embedding dimension \(d_e\) via a linear adapter \(\mathbf{W}_{\mathrm{adapt}} \in \mathbb{R}^{d_{\mathrm{feat}} \times d_e}\). To aggregate disparate topological signals into a compact sample-level representation, we prepend \(C\) learnable class tokens \(\mathbf{C} = [\mathbf{c}_1, \dots, \mathbf{c}_C]^\top \in \mathbb{R}^{C \times d_e}\) (\(C{=}4\) by default) to the feature sequence:
\begin{equation}
\mathbf{Z}_k^{(0)} = \big[ \mathbf{C} \;;\; \mathbf{X}_k \mathbf{W}_{\mathrm{adapt}} \big] \in \mathbb{R}^{(C + F) \times d_e}.
\end{equation}
The sequence \(\mathbf{Z}_k^{(0)}\) is then processed by a stack of \(L_{\mathrm{feat}}\) Transformer encoder layers equipped with Rotary Position Embedding (RoPE)~\citep{rope}. At each layer \(l \in \{1, \dots, L_{\mathrm{feat}}\}\), the multi-head self-attention operation is defined as:
\begin{align}
\mathbf{Q}_k^{(l)} &= \mathbf{R}_{\Theta}\big(\mathbf{Z}_k^{(l-1)} \mathbf{W}_Q^{(l)}\big), \quad
\mathbf{K}_k^{(l)} = \mathbf{R}_{\Theta}\big(\mathbf{Z}_k^{(l-1)} \mathbf{W}_K^{(l)}\big), \quad
\mathbf{V}_k^{(l)} = \mathbf{Z}_k^{(l-1)} \mathbf{W}_V^{(l)}, \label{eq:feat_attn_qkv} \\
\mathbf{A}_{k, \mathrm{feat}}^{(l)} &= \mathrm{Softmax}\left( \frac{\mathbf{Q}_k^{(l)} (\mathbf{K}_k^{(l)})^\top}{\sqrt{d_k}} \right) \mathbf{V}_k^{(l)}, \label{eq:feat_attn_formula}
\end{align}
followed by standard residual connections, LayerNorm, and feed-forward networks:
\begin{equation}
\mathbf{Z}_k^{(l)} = \mathbf{Z}_k^{(l-1)} + \mathrm{Dropout}\big(\mathbf{A}_{k, \mathrm{feat}}^{(l)} \mathbf{W}_O^{(l)}\big), \quad
\mathbf{Z}_k^{(l)} \leftarrow \mathbf{Z}_k^{(l)} + \mathrm{FFN}\big(\mathrm{LayerNorm}(\mathbf{Z}_k^{(l)})\big).
\end{equation}
In the final layer \(L_{\mathrm{feat}}\), the \(C\) updated class tokens act as queries attending over the full feature context. The resulting class representations \([\mathbf{z}_{k, 1}^{(L_{\mathrm{feat}})}, \dots, \mathbf{z}_{k, C}^{(L_{\mathrm{feat}})}]\) are normalized with LayerNorm and concatenated to construct the final sample-level representation vector:
\begin{equation}
\mathbf{u}_k = \mathrm{Concat}\Big(\mathbf{z}_{k, 1}^{(L_{\mathrm{feat}})}, \, \mathbf{z}_{k, 2}^{(L_{\mathrm{feat}})}, \, \dots, \, \mathbf{z}_{k, C}^{(L_{\mathrm{feat}})}\Big) \in \mathbb{R}^{d_{\mathrm{icl}}},
\label{eq:feat_attn_concat}
\end{equation}
where \(d_{\mathrm{icl}} = C \cdot d_e\). This effectively condenses the multi-faceted graph structure of each triple into a fixed-dimensional latent embedding.

\paragraph{Sample-Level Attention (Dataset-wise In-Context Learning).}
Once every triple \(k\) is mapped to \(\mathbf{u}_k \in \mathbb{R}^{d_{\mathrm{icl}}}\), the sample-level Transformer models relational dependencies between the query triple and the support context triples.

First, ground-truth binary labels \(y_i \in \{0, 1\}\) for context triples are mapped into the latent space via a linear target encoder \(\mathbf{W}_y \in \mathbb{R}^{d_{\mathrm{icl}} \times 1}\) with bias \(\mathbf{b}_y\) and added to the context representations, while query representations remain unaltered:
\begin{equation}
\tilde{\mathbf{u}}_i = 
\begin{cases}
\mathbf{u}_i + \mathbf{W}_y y_i + \mathbf{b}_y, & \text{if } i \in \{1, \dots, m\} \quad (\text{context triple in } \mathcal{C}_r) \\
\mathbf{u}_i, & \text{if } i \in \{m+1, \dots, m+N\} \quad (\text{candidate query triple})
\end{cases}
\label{eq:sample_label_injection}
\end{equation}
The combined sequence is formed as:
\begin{equation}
\mathbf{H}^{(0)} = \big[\tilde{\mathbf{u}}_1, \, \dots, \, \tilde{\mathbf{u}}_m, \, \tilde{\mathbf{u}}_{q, 1}, \, \dots, \, \tilde{\mathbf{u}}_{q, N}\big]^\top \in \mathbb{R}^{(m + N) \times d_{\mathrm{icl}}}.
\end{equation}
The sequence passes through \(L_{\mathrm{icl}}\) Transformer layers. At each layer, attention queries are computed from all rows, while keys and values are computed only from the context rows. Thus, context examples perform bidirectional self-attention, and each candidate query performs cross-attention exclusively to the context pool. Candidate query rows contribute no keys or values, so they attend neither to themselves nor to other candidate queries. Each candidate has its own attention query vector and retrieves a candidate-specific weighted combination of the shared context values.
Formally, let \(\mathbf{H}_{\mathrm{ctx}}^{(l-1)} \in \mathbb{R}^{m \times d_{\mathrm{icl}}}\) denote the context representations and \(\mathbf{h}_{q, j}^{(l-1)} \in \mathbb{R}^{d_{\mathrm{icl}}}\) denote a query triple representation at layer \(l-1\). The sample-level cross-attention is computed as:
\begin{align}
\mathbf{q}_{q, j}^{(l)} &= \mathbf{h}_{q, j}^{(l-1)} \mathbf{W}_{Q, \mathrm{icl}}^{(l)}, \quad
\mathbf{K}_{\mathrm{ctx}}^{(l)} = \mathbf{H}_{\mathrm{ctx}}^{(l-1)} \mathbf{W}_{K, \mathrm{icl}}^{(l)}, \quad
\mathbf{V}_{\mathrm{ctx}}^{(l)} = \mathbf{H}_{\mathrm{ctx}}^{(l-1)} \mathbf{W}_{V, \mathrm{icl}}^{(l)}, \\
\mathbf{a}_{q, j, \mathrm{sample}}^{(l)} &= \mathrm{Softmax}\left( \frac{\mathbf{q}_{q, j}^{(l)} (\mathbf{K}_{\mathrm{ctx}}^{(l)})^\top}{\sqrt{d_{\mathrm{icl}}}} \odot \mathbf{\Gamma}\big(\mathbf{q}_{q, j}^{(l)}, m\big) \right) \mathbf{V}_{\mathrm{ctx}}^{(l)}, \label{eq:sample_attn_formula}
\end{align}
where \(\mathbf{\Gamma}(\mathbf{q}, m) = \mathrm{MLP}_{\mathrm{base}}(\log m) \cdot \big(1 + \tanh(\mathrm{MLP}_{\mathrm{query}}(\mathbf{q}))\big)\) represents TabICL's query-aware scalable softmax (QASSMax) modulation~\citep{tabicl}, preventing attention entropy collapse across varied context sizes \(m\).
The attention output is combined with that row's incoming representation through a residual connection, preserving candidate-specific information alongside the retrieved evidence. Layer normalization and feed-forward transformations are applied within each row. After \(L_{\mathrm{icl}}\) layers, only the candidate query rows are passed to the scoring head. Each contextualized query representation \(\mathbf{h}_{q, j}^{(L_{\mathrm{icl}})}\) is projected by a lightweight two-layer MLP to produce the scalar plausibility prediction:
\begin{equation}
s(h, r, t_j) = \mathbf{w}_2^\top \, \mathrm{GELU}\big(\mathbf{W}_1 \mathbf{h}_{q, j}^{(L_{\mathrm{icl}})} + \mathbf{b}_1\big) + b_2 \in \mathbb{R}.
\label{eq:pfn_score_head}
\end{equation}

\subsection{Theoretical Inference Complexity Analysis}
\label{app:complexity_analysis}

We separate graph encoding, support preparation, and candidate scoring. Let \(n=|\mathcal{V}|\) be the number of candidate entities, \(m=|\mathcal{C}_r|\) the support size, and \(p=C+F\) the number of feature tokens including class tokens. Write \(D=d_{\mathrm{icl}}\), and assume the feed-forward widths scale proportionally to their respective Transformer embedding widths. We likewise take the QASSMax modulation-network widths to be \(\mathcal{O}(D)\), so their per-query cost is included in the \(\mathcal{O}(D^2)\) dense-computation term.

\paragraph{Graph encoding.}
Let \(T_{\mathrm{rel}}\) be the cost of preparing query-conditioned relation representations on the visible graph. Let \(T_{\mathrm{enc}}(h,r)\) be the cost of one head-conditioned entity-encoder pass, including local-context readout and final entity representations. One such pass serves all candidate tails for a given head and relation. Each context triple also requires representations conditioned on its own head. If \(\mathcal{H}_r\) denotes the distinct context heads, the cold graph-encoding cost, grouping identical head--relation pairs, is
\begin{equation}
    T_{\mathrm{graph}}
    =T_{\mathrm{rel}}+
    \sum_{u\in\{h\}\cup\mathcal{H}_r}T_{\mathrm{enc}}(u,r).
\end{equation}
This separates context graph encoding from the query pass. The exact encoder cost depends on the message and update operators; the support representations can be reused across candidate chunks once computed. This expression assumes that static graph structures and retrieval indices have already been built, matching the preprocessing convention of the empirical benchmark below.

\paragraph{Feature encoding and support prefill.}
For each triple, input projection and the feature Transformer cost
\begin{equation}
    T_{\mathrm{feat}}
    =\mathcal{O}\!\left(
    Fd_{\mathrm{feat}}d_e+
    L_{\mathrm{feat}}(p^2d_e+pd_e^2)\right).
\end{equation}
The first term is the feature projection, while the two layer terms account for attention products and dense projections/feed-forward operations. Let \(T_{\mathrm{adapt}}\) denote the per-triple cost of constructing and projecting the structural features in Eq.~\ref{eq:context_adapter_fuse}. Preparing all support examples costs \(m(T_{\mathrm{adapt}}+T_{\mathrm{feat}})\). Their sample-level self-attention and KV-cache preparation additionally cost
\begin{equation}
    T_{\mathrm{prefill}}
    =\mathcal{O}\!\left(L_{\mathrm{icl}}(m^2D+mD^2)\right).
\end{equation}

\paragraph{All-entity candidate scoring.}
Each candidate independently encodes its features and attends to the cached support. Across all \(n\) candidates, feature construction and encoding cost \(n(T_{\mathrm{adapt}}+T_{\mathrm{feat}})\); sample-level attention, dense projections, and feed-forward operations cost
\begin{equation}
    T_{\mathrm{query}}
    =\mathcal{O}\!\left(nL_{\mathrm{icl}}(mD+D^2)+nDd_{\mathrm{head}}\right),
\end{equation}
where \(d_{\mathrm{head}}\) is the hidden width of the final two-layer scoring MLP. Combining these terms gives
\begin{equation}
\begin{aligned}
    T_{\mathrm{infer}}
    ={}&T_{\mathrm{retrieve}}+T_{\mathrm{graph}}
    +(m+n)(T_{\mathrm{adapt}}+T_{\mathrm{feat}})\\
    &+T_{\mathrm{prefill}}+T_{\mathrm{query}},
\end{aligned}
\label{eq:total_infer_complexity}
\end{equation}
where \(T_{\mathrm{retrieve}}\) accounts for retrieving and sampling the support examples. For fixed model dimensions and support size, the candidate-scoring stage scales linearly with \(n\); graph encoding and support preparation are accounted for separately.

\paragraph{Acceleration via context caching and chunking.}
The support set is held fixed while scoring all candidates for a query. We encode its feature representations once and run support self-attention once per layer, storing the resulting keys and values. Because support states do not attend to candidate queries, these cached states are reusable across candidate chunks. Caching amortizes the support prefill over the \(\lceil n/N_c\rceil\) chunks, where \(N_c\) is the candidate chunk size.

The per-layer support KV cache requires \(\mathcal{O}(mD)\) storage, totaling \(\mathcal{O}(L_{\mathrm{icl}}mD)\) across layers. Dense attention weights, when materialized, occupy \(\mathcal{O}(m^2)\) space during support prefill and \(\mathcal{O}(N_cm)\) during candidate cross-attention for a fixed number of heads. Feature attention uses \(\mathcal{O}(B_{\mathrm{feat}}p^2)\) attention-weight storage for a feature batch of \(B_{\mathrm{feat}}\) triples. These workspaces coexist with model weights, graph-encoder state, and projected features. The attention mask makes candidate queries independent of one another, allowing cross-attention scoring without an \(n\times n\) candidate-attention matrix; chunking limits the number processed simultaneously. The end-to-end peak memory is measured below.

\subsection{Empirical Efficiency Experiment}
\label{app:efficiency_benchmark}

We compare the inference cost of KGPFN against ULTRA~\citep{ultra}, TRIX~\citep{trix}, MOTIF~\citep{motif}, FLOCK, and KGICL~\citep{kgicl} on FB15k-237~\citep{fb15k}.
Tail-only all-entity ranking queries \((h,r,?)\) are timed in groups of \(B{=}64\), reporting the mean across three non-overlapping groups after one warmup group. Graph-level preprocessing (e.g., relation graphs, index building) is excluded, while per-query costs (random walks, context retrieval/encoding, scoring) are included. For method settings, KGPFN uses \(20\) positive and \(80\) negative context triples; KGICL constructs \(3\) hop-\(2\) prompt graphs per tail relation; and FLOCK runs with its standard \(128\)-walk configuration under an evaluation microbatch of \(8\).

\begin{table}[H]
\centering
\caption{Inference efficiency comparison on FB15k-237.}
\label{tab:inference_efficiency}
\small
\begin{tabular}{lrr}
\toprule
Model & ms/query & Peak GPU (GB) \\
\midrule
ULTRA  & 0.95 & 2.34 \\
TRIX   & 5.00 & 1.46 \\
MOTIF  & 18.7 & 2.48 \\
KGICL & 61.8 & 24.9 \\
KGPFN  & 62.1 & 2.78 \\
FLOCK  & 491  & 50.9 \\
\bottomrule
\end{tabular}
\end{table}

As shown in Table~\ref{tab:inference_efficiency}, ULTRA, TRIX, and MOTIF are fastest after graph-level preprocessing, using direct structure-based scoring. KGICL, KGPFN, and FLOCK additionally retrieve or sample graph evidence for inference. Among these three methods, KGPFN achieves latency close to KGICL with substantially lower peak memory, and is faster and more memory-efficient than FLOCK in this benchmark. KGPFN amortizes all-tail scoring through a shared structure encoding and cached PFN context.

\section{More Experiments}
\subsection{Datasets}
\label{app:dataset}
\paragraph{Transductive datasets (16).}
The transductive split contains graphs where test entities appear in the training graph.
It includes FB15k-237~\cite{fb15k} and its sparse subsets (FB15k-237-10\%, FB15k-237-20\%, FB15k-237-50\%), WN18RR~\cite{wn18rr}, CoDEx-Small~\cite{codex}, CoDEx-Medium~\cite{codex}, CoDEx-Large~\cite{codex}, NELL-995~\cite{nell}, ConceptNet100k~\cite{conceptnet}, DBpedia100k~\cite{dbpedia}, YAGO3-10~\cite{yago}, AristoV4~\cite{aristov4}, Hetionet~\cite{hetionet}, WDsinger~\cite{ultra}, and NELL23k. The details for these datasets are reported in Table~\ref{tab:transductive_datasets}.

\paragraph{Inductive datasets --- new entities, no new relations (18).}
These datasets introduce unseen entities at inference time while keeping the relation vocabulary fixed.
They comprise the GraIL benchmark graphs~\cite{grial}: FB15k-237-Inductive (v1--v4), WN18RR-Inductive (v1--v4), and NELL-Inductive (v1--v4); the ILPC 2022 graphs (small and large)~\cite{galkin2022open}; and 4 INDIGO graphs~\cite{indigo,indigo2}: HM-1k, HM-3k, HM-5k, and HM-Indigo. The details for these datasets are reported in Table~\ref{tab:inductive_entity_datasets}.

\paragraph{Fully inductive datasets --- new entities and new relations (23).}
These datasets require generalization to both unseen entities and unseen relation types.
They include the InGram graphs~\cite{ingram}: NL-Ingram (splits 0/25/50/75/100), FB-Ingram (25/50/75/100), and WK-Ingram (25/50/75/100); the MTDEA WikiTopics graphs~\cite{zhou2023ood}: WikiTopics-MT1 (health, tax), WikiTopics-MT2 (org, sci), WikiTopics-MT3 (art, infra), and WikiTopics-MT4 (sci, health); and two additional MTDEA graphs, Metafam and FBNELL. The details for these datasets are reported in Table~\ref{tab:inductive_entity_relation_datasets}.

\begin{table}[t]
\centering
\caption{16 Transductive datasets used in the experiments.}
\label{tab:transductive_datasets}
\begin{tabular}{lrrrrr}
\toprule
Dataset & Entities & Rels & Train & Valid & Test \\
\midrule
CoDEx Small     & 2034   & 42   & 32888   & 1827   & 1828 \\
CoDEx Medium    & 17050  & 51   & 185584  & 10310  & 10311 \\
CoDEx Large     & 77951  & 69   & 551193  & 30622  & 30622 \\
WDsinger        & 10282  & 135  & 16142   & 2163   & 2203 \\
FB15k237\_10    & 11512  & 237  & 27211   & 15624  & 18150 \\
FB15k237\_20    & 13166  & 237  & 54423   & 16963  & 19776 \\
FB15k237\_50    & 14149  & 237  & 136057  & 17449  & 20324 \\
FB15k237        & 14541  & 237  & 272115  & 17535  & 20466 \\
NELL23k         & 22925  & 200  & 25445   & 4961   & 4952 \\
WN18RR          & 40943  & 11   & 86835   & 3034   & 3134 \\
AristoV4        & 44949  & 1605 & 242567  & 20000  & 20000 \\
Hetionet        & 45158  & 24   & 2025177 & 112510 & 112510 \\
NELL995         & 74536  & 200  & 149678  & 543    & 2818 \\
ConceptNet100k  & 78334  & 34   & 100000  & 1200   & 1200 \\
DBpedia100k     & 99604  & 470  & 597572  & 50000  & 50000 \\
YAGO310         & 123182 & 37   & 1079040 & 5000   & 5000 \\
\bottomrule
\end{tabular}
\end{table}

\begin{table}[H]
\centering
\caption{18 Inductive entity ($e$) datasets used in the experiments. }
\label{tab:inductive_entity_datasets}
\small
\setlength{\tabcolsep}{4pt}
\begin{tabular}{lrrrrrrrrr}
\toprule
\multirow{2}{*}{Dataset} & \multirow{2}{*}{Rels}
& \multicolumn{2}{c}{Training Graph}
& \multicolumn{3}{c}{Validation Graph}
& \multicolumn{3}{c}{Test Graph} \\
\cmidrule(lr){3-4}
\cmidrule(lr){5-7}
\cmidrule(lr){8-10}
& & Entities & Triples & Entities & Triples & Valid & Entities & Triples & Test \\
\midrule
FB v1         & 180 & 1594  & 4245   & 1594  & 4245   & 489   & 1093  & 1993   & 411 \\
FB v2         & 200 & 2608  & 9739   & 2608  & 9739   & 1166  & 1660  & 4145   & 947 \\
FB v3        & 215 & 3668  & 17986  & 3668  & 17986  & 2194  & 2501  & 7406   & 1731 \\
FB v4        & 219 & 4707  & 27203  & 4707  & 27203  & 3352  & 3051  & 11714  & 2840 \\
WN v1        & 9   & 2746  & 5410   & 2746  & 5410   & 630   & 922   & 1618   & 373 \\
WN v2        & 10  & 6954  & 15262  & 6954  & 15262  & 1838  & 2757  & 4011   & 852 \\
WN v3        & 11  & 12078 & 25901  & 12078 & 25901  & 3097  & 5084  & 6327   & 1143 \\
WN v4         & 9   & 3861  & 7940   & 3861  & 7940   & 934   & 7084  & 12334  & 2823 \\
NELL v1      & 14  & 3103  & 4687   & 3103  & 4687   & 414   & 225   & 833    & 201 \\
NELL v2      & 88  & 2564  & 8219   & 2564  & 8219   & 922   & 2086  & 4586   & 935 \\
NELL v3      & 142 & 4647  & 16393  & 4647  & 16393  & 1851  & 3566  & 8048   & 1620 \\
NELL v4      & 76  & 2092  & 7546   & 2092  & 7546   & 876   & 2795  & 7073   & 1447 \\
ILPC Small  & 48  & 10230 & 78616  & 6653  & 20960  & 2908  & 6653  & 20960  & 2902 \\
ILPC Large  & 65  & 46626 & 202446 & 29246 & 77044  & 10179 & 29246 & 77044  & 10184 \\
HM 1k        & 11  & 36237 & 93364  & 36311 & 93364  & 1771  & 9899  & 18638  & 476 \\
HM 3k       & 11  & 32118 & 71097  & 32250 & 71097  & 1201  & 19218 & 38285  & 1349 \\
HM 5k       & 11  & 28601 & 57601  & 28744 & 57601  & 900   & 23792 & 48425  & 2124 \\
IndigoBM    & 229 & 12721 & 121601 & 12797 & 121601 & 14121 & 14775 & 250195 & 14904 \\
\bottomrule
\end{tabular}
\end{table}

\begin{table}[H]
\centering
\caption{23 Inductive entity and relation ($e, r$) datasets used in the experiments.}
\label{tab:inductive_entity_relation_datasets}
\small
\setlength{\tabcolsep}{3pt}
\begin{tabular}{lrrrrrrrrrrr}
\toprule
\multirow{2}{*}{Dataset}
& \multicolumn{3}{c}{Training Graph}
& \multicolumn{4}{c}{Validation Graph}
& \multicolumn{4}{c}{Test Graph} \\
\cmidrule(lr){2-4}
\cmidrule(lr){5-8}
\cmidrule(lr){9-12}
& Entities & Rels & Triples
& Entities & Rels & Triples & Valid
& Entities & Rels & Triples & Test \\
\midrule
FB-25        & 5190  & 163 & 91571 & 4097  & 216 & 17147 & 5716 & 4097  & 216 & 17147 & 5716 \\
FB-50        & 5190  & 153 & 85375 & 4445  & 205 & 11636 & 3879 & 4445  & 205 & 11636 & 3879 \\
FB-75        & 4659  & 134 & 62809 & 2792  & 186 & 9316  & 3106 & 2792  & 186 & 9316  & 3106 \\
FB-100       & 4659  & 134 & 62809 & 2624  & 77  & 6987  & 2329 & 2624  & 77  & 6987  & 2329 \\
WK-25        & 12659 & 47  & 41873 & 3228  & 74  & 3391  & 1130 & 3228  & 74  & 3391  & 1131 \\
WK-50        & 12022 & 72  & 82481 & 9328  & 93  & 9672  & 3224 & 9328  & 93  & 9672  & 3225 \\
WK-75        & 6853  & 52  & 28741 & 2722  & 65  & 3430  & 1143 & 2722  & 65  & 3430  & 1144 \\
WK-100       & 9784  & 67  & 49875 & 12136 & 37  & 13487 & 4496 & 12136 & 37  & 13487 & 4496 \\
NL-0         & 1814  & 134 & 7796  & 2026  & 112 & 2287  & 763  & 2026  & 112 & 2287  & 763 \\
NL-25        & 4396  & 106 & 17578 & 2146  & 120 & 2230  & 743  & 2146  & 120 & 2230  & 744 \\
NL-50        & 4396  & 106 & 17578 & 2335  & 119 & 2576  & 859  & 2335  & 119 & 2576  & 859 \\
NL-75        & 2607  & 96  & 11058 & 1578  & 116 & 1818  & 606  & 1578  & 116 & 1818  & 607 \\
NL-100       & 1258  & 55  & 7832  & 1709  & 53  & 2378  & 793  & 1709  & 53  & 2378  & 793 \\
\midrule
Metafam      & 1316  & 28  & 13821 & 1316  & 28  & 13821 & 590  & 656   & 28  & 7257  & 184 \\
FBNELL       & 4636  & 100 & 10275 & 4636  & 100 & 10275 & 1055 & 4752  & 183 & 10685 & 597 \\
Wiki MT1 tax & 10000 & 10  & 17178 & 10000 & 10  & 17178 & 1908 & 10000 & 9   & 16526 & 1834 \\
Wiki MT1 health & 10000 & 7 & 14371 & 10000 & 7 & 14371 & 1596 & 10000 & 7 & 14110 & 1566 \\
Wiki MT2 org & 10000 & 10 & 23233 & 10000 & 10 & 23233 & 2581 & 10000 & 11 & 21976 & 2441 \\
Wiki MT2 sci & 10000 & 16 & 16471 & 10000 & 16 & 16471 & 1830 & 10000 & 16 & 14852 & 1650 \\
Wiki MT3 art & 10000 & 45 & 27262 & 10000 & 45 & 27262 & 3026 & 10000 & 45 & 28023 & 3113 \\
Wiki MT3 infra & 10000 & 24 & 21990 & 10000 & 24 & 21990 & 2443 & 10000 & 27 & 21646 & 2405 \\
Wiki MT4 sci & 10000 & 42 & 12576 & 10000 & 42 & 12576 & 1397 & 10000 & 42 & 12516 & 1388 \\
Wiki MT4 health & 10000 & 21 & 15539 & 10000 & 21 & 15539 & 1725 & 10000 & 20 & 15337 & 1703 \\
\bottomrule
\end{tabular}
\end{table}

\subsection{Training Details}
\label{app:training}
\textbf{Pretraining.}
We pretrain on FB15K-237~\citep{fb15k}, WN18RR~\citep{wn18rr}, and CoDEx-Medium~\citep{codex} with AdamW, a learning rate of \(5\times10^{-5}\), and weight decay \(10^{-6}\). The batch size is 32, and each positive query is paired with 512 negative samples.
The keep ratio \(\rho\) is drawn uniformly from \([0.5,0.9]\). For each relation, the perturbation \(\delta\) is drawn uniformly from \([0,0.08]\) and added to \(\rho\); this perturbed ratio is the rate at which that relation's edges are retained in the support graph.
Every 10 steps we resample the training episode: one of the three graphs is drawn first, and a support graph together with a query graph is then sampled from it. We count 100 steps as one epoch and train for 100 epochs.

\textbf{Fine-tuning.}
Fine-tuning also uses AdamW and has two stages on each target dataset's training split. We first train the structure encoder for 3 epochs with learning rate \(1\times10^{-5}\). We then freeze the encoder and train the PFN for 20 epochs, using the same optimizer settings as pretraining. At evaluation, context construction and graph encoding use the visible support edges of the target graph, excluding held-out prediction targets.

\paragraph{Global-context sampling.}
The default context budget is 20 positive and 80 negative triples. If fewer than 20 positive examples are available for the relation, the missing positive slots are reassigned to negative examples. For a given head \(h\) and relation \(r\), the negative sampler targets approximately 80\% of its tail candidates from the head's two-hop neighborhood, shuffling the neighbors before selection. The remaining approximately 20\% are sampled uniformly from the graph's entity set. A shortfall in neighborhood candidates is filled by uniform sampling; if further candidates are needed, the sampler scans a shuffled entity list. At each stage, it rejects triples already present in the visible support graph, the current query triple, and positive triples selected for the context.

\subsection{More results}
\label{app:moreresults}
This section reports per-dataset MRR comparisons on the 57-graph benchmark. Figures~\ref{fig:detailed_transductive}, \ref{fig:detailed_inductive}, and \ref{fig:detailed_full_inductive} show competitive performance without fine-tuning across the three settings.

\begin{figure}[H]
  \centering
  \begin{subfigure}{\linewidth}
    \centering
    \includegraphics[width=0.7\linewidth]{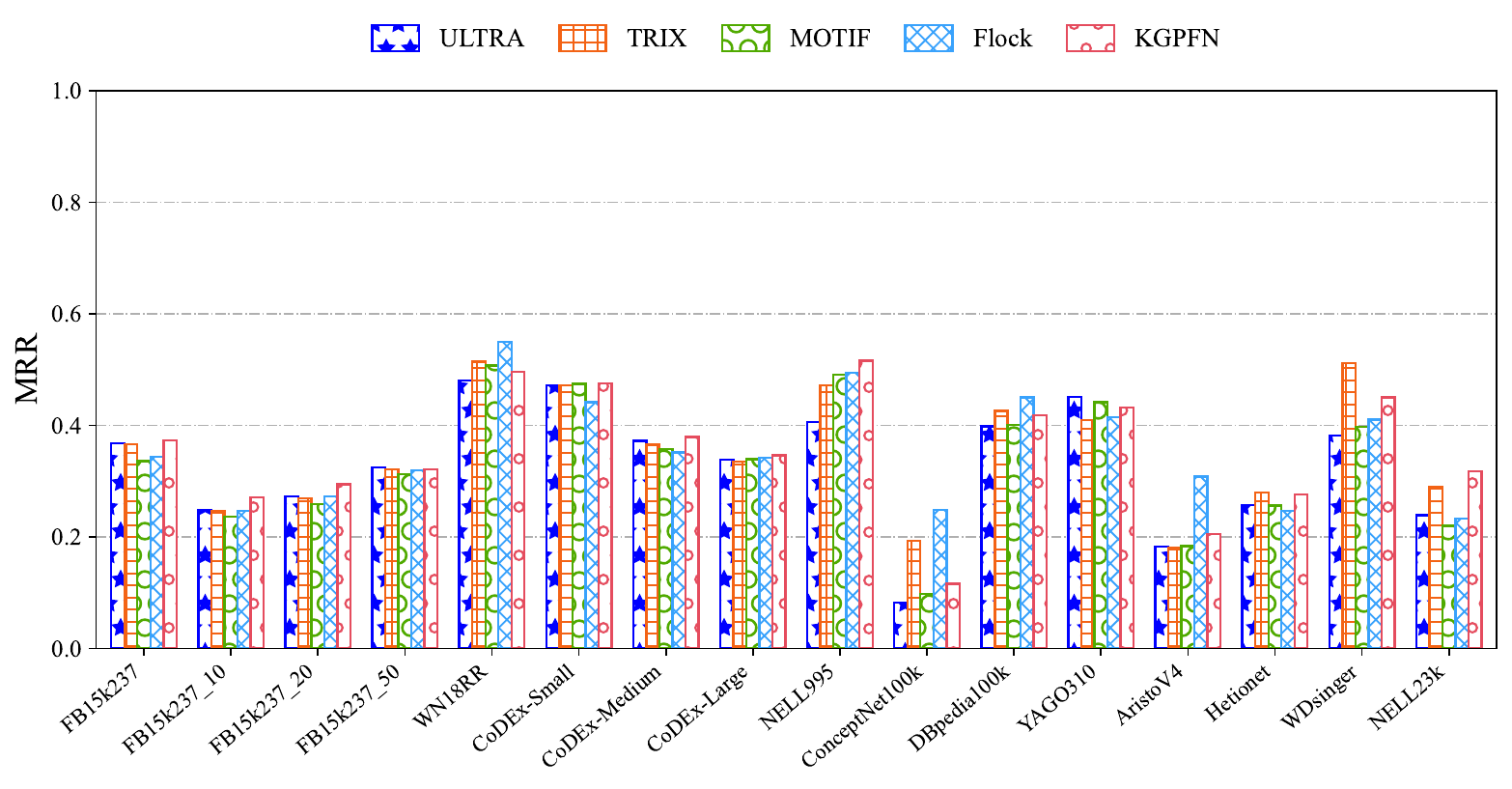}
    \caption{16 transductive KGs.}
    \label{fig:detailed_transductive}
  \end{subfigure}

  \vspace{0.6em}
  \begin{subfigure}{\linewidth}
    \centering
    \includegraphics[width=0.7\linewidth]{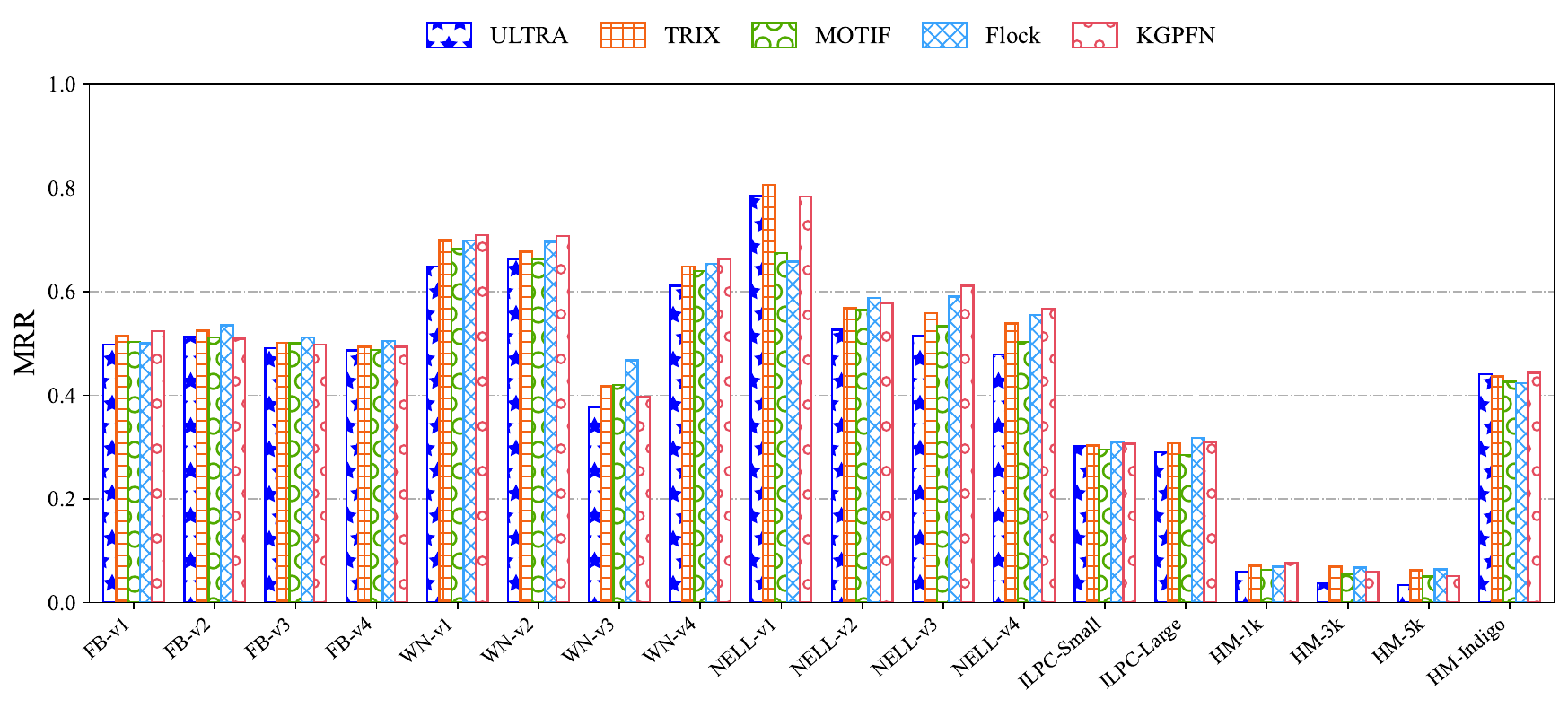}
    \caption{18 inductive KGs.}
    \label{fig:detailed_inductive}
  \end{subfigure}

  \vspace{0.6em}
  \begin{subfigure}{\linewidth}
    \centering
    \includegraphics[width=0.7\linewidth]{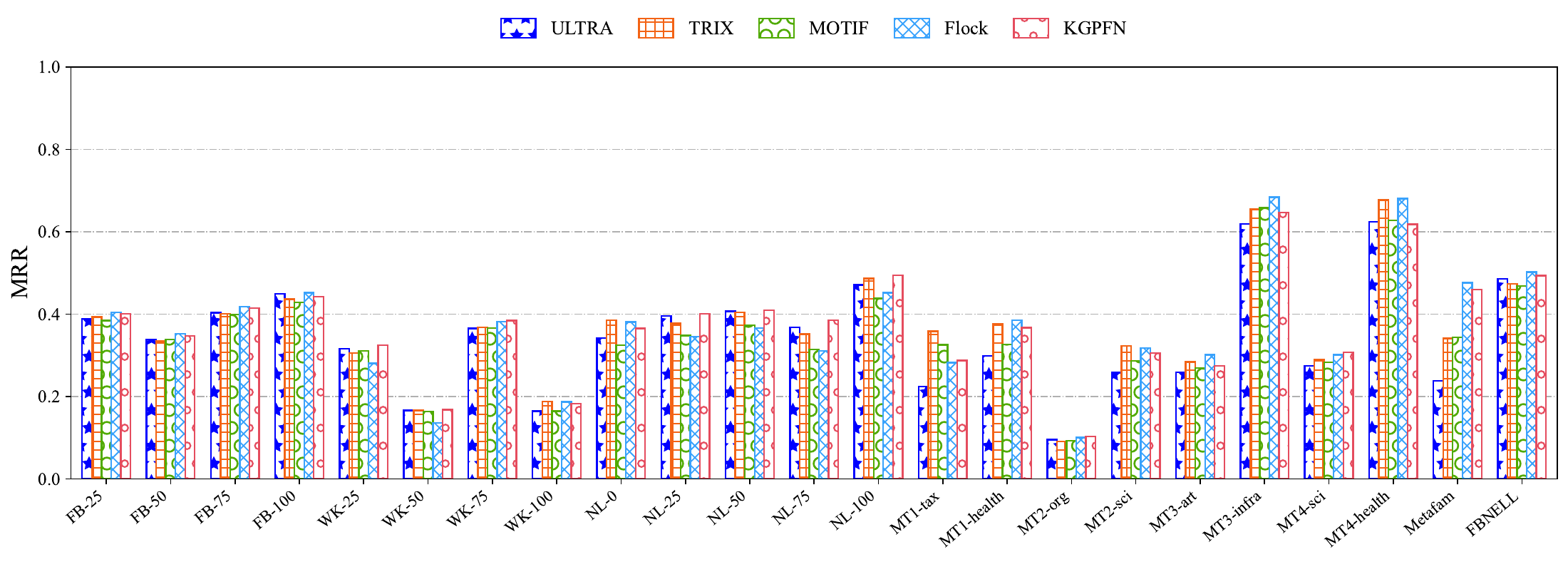}
    \caption{23 fully inductive KGs.}
    \label{fig:detailed_full_inductive}
  \end{subfigure}
  \caption{MRR without fine-tuning.}
  \label{fig:detailed_zeroshot}
\end{figure}

Fig.~\ref{fig:detailed_transductive_ft}, Fig.~\ref{fig:detailed_inductive_ft} and Fig.~\ref{fig:detailed_full_inductive_ft} show the same per-dataset comparison after fine-tuning.

\begin{figure}[H]
  \centering
  \begin{subfigure}{\linewidth}
    \centering
    \includegraphics[width=0.7\linewidth]{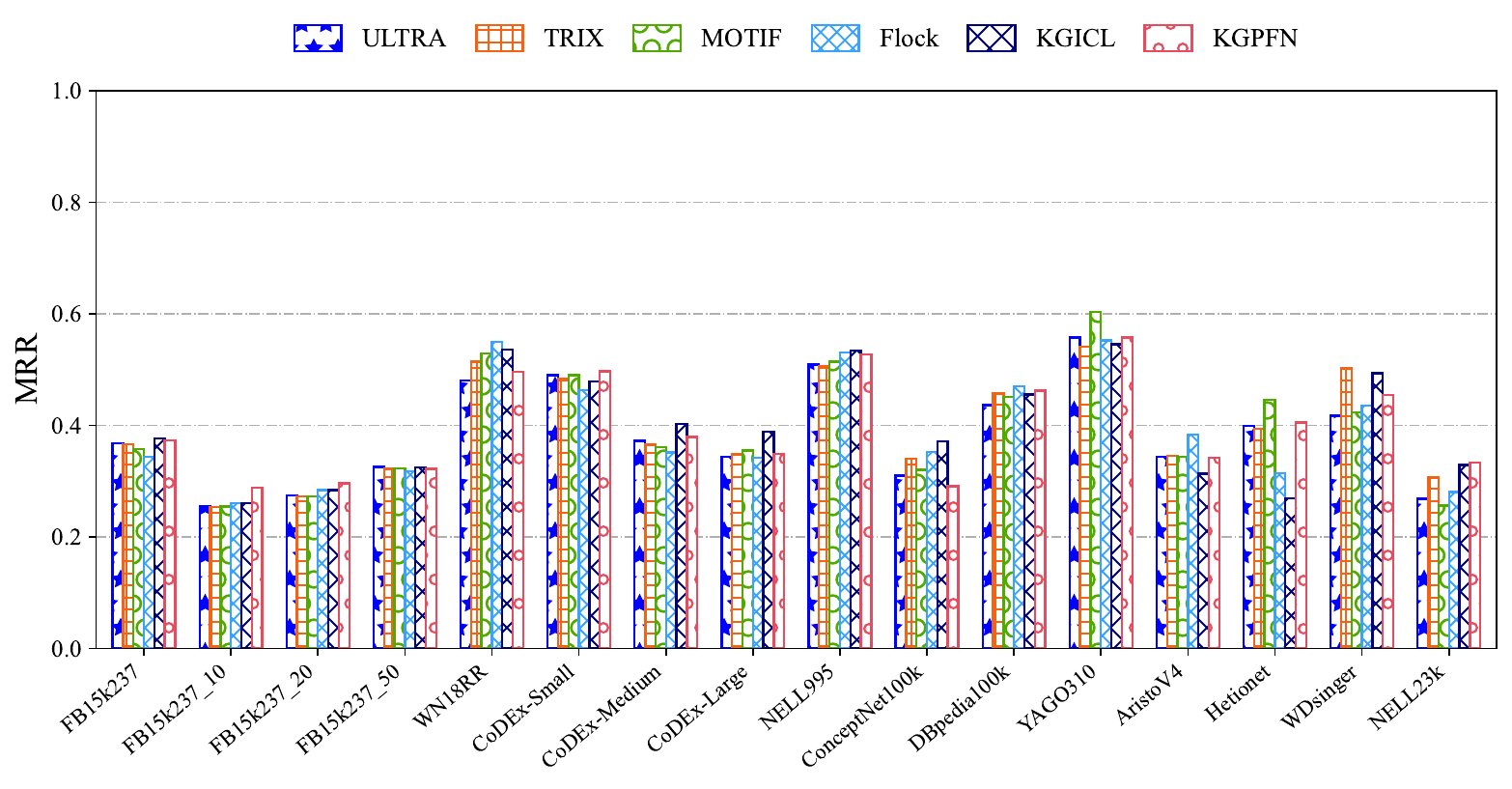}
    \caption{16 transductive KGs.}
    \label{fig:detailed_transductive_ft}
  \end{subfigure}

  \vspace{0.6em}
  \begin{subfigure}{\linewidth}
    \centering
    \includegraphics[width=0.7\linewidth]{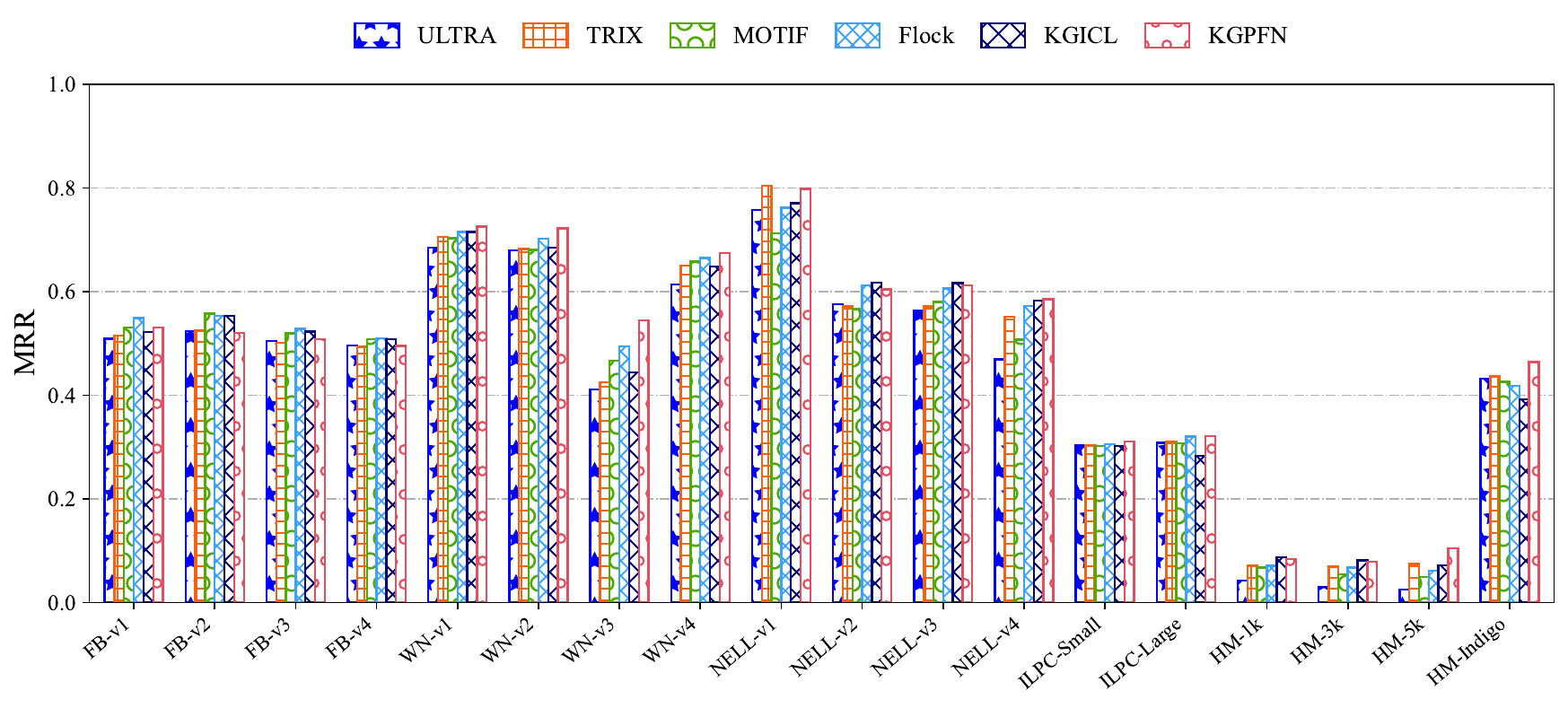}
    \caption{18 inductive KGs.}
    \label{fig:detailed_inductive_ft}
  \end{subfigure}

  \vspace{0.6em}
  \begin{subfigure}{\linewidth}
    \centering
    \includegraphics[width=0.7\linewidth]{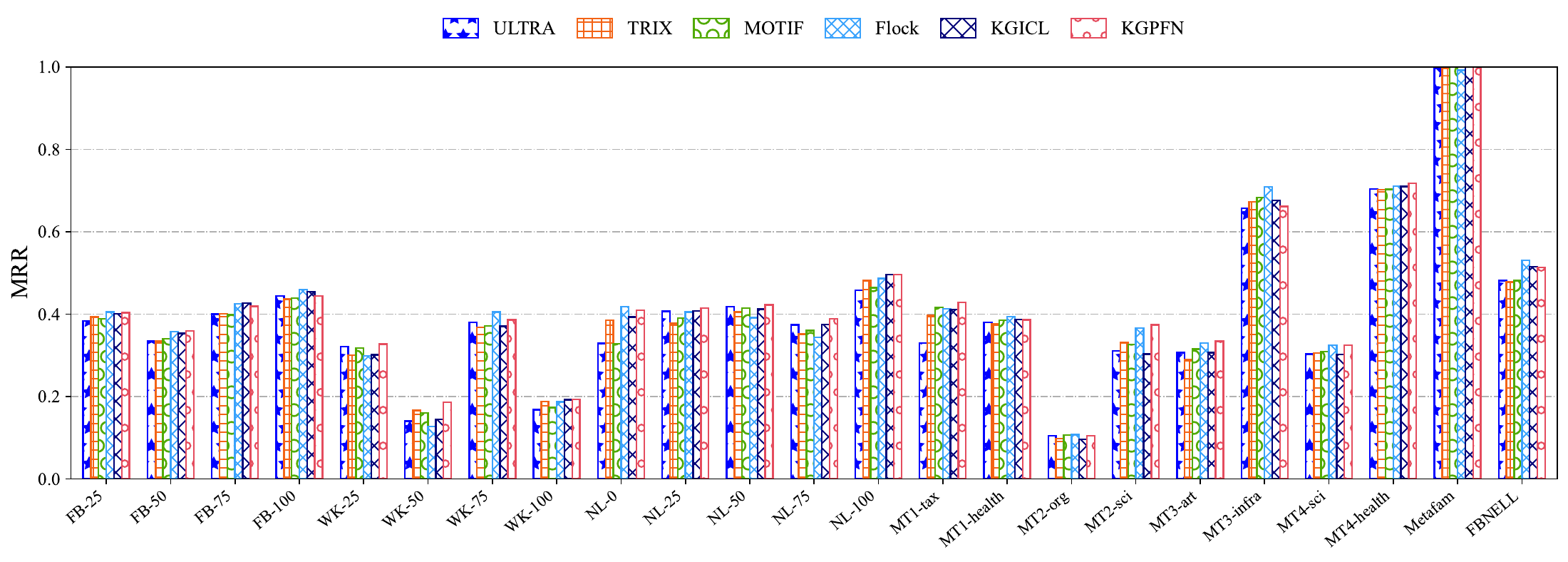}
    \caption{23 fully inductive KGs.}
    \label{fig:detailed_full_inductive_ft}
  \end{subfigure}
  \caption{MRR after fine-tuning.}
  \label{fig:detailed_finetune}
\end{figure}

Here, we report KGPFN's per-dataset mean and standard deviation over five inference runs with different context-sampling seeds, using a fixed checkpoint for each dataset and evaluation setting. W/o FT denotes inference without fine-tuning. The three pretraining graphs appear in Table~\ref{tab:pretraining}, the other 13 transductive graphs in Table~\ref{tab:transductive}, and the inductive graphs in Tables~\ref{tab:inductive} and \ref{tab:full_inductive}. The main table's 16-graph transductive average includes both transductive subsets, and its 57-graph overall average includes every graph once.
Per-dataset results are rounded to three decimal places; averaging these displayed values can introduce small rounding differences relative to averaging the unrounded results.

\begin{table}[H]
\centering
\caption{Results on the three pretraining datasets (MRR / Hits@10).}
\label{tab:pretraining}
\small
\begin{tabular}{lcc|cc}
\toprule
 & \multicolumn{2}{c|}{\textbf{KGPFN (w/o FT)}} & \multicolumn{2}{c}{\textbf{KGPFN (fine-tuned)}} \\
\textbf{Dataset} & \textbf{MRR} & \textbf{Hits@10} & \textbf{MRR} & \textbf{Hits@10} \\
\midrule
FB15k237        & 0.373$\pm$0.004 & 0.570$\pm$0.003 & 0.373$\pm$0.004 & 0.570$\pm$0.003 \\
WN18RR          & 0.496$\pm$0.026 & 0.634$\pm$0.029 & 0.496$\pm$0.026 & 0.634$\pm$0.029 \\
CoDExMedium     & 0.379$\pm$0.004 & 0.535$\pm$0.004 & 0.379$\pm$0.004 & 0.535$\pm$0.004 \\
\bottomrule
\end{tabular}
\end{table}

\begin{table}[H]
\centering
\caption{Results on the 13 transductive datasets outside the three-graph pretraining set (MRR / Hits@10).}
\label{tab:transductive}
\small
\begin{tabular}{lcc|cc}
\toprule
 & \multicolumn{2}{c|}{\textbf{KGPFN (w/o FT)}} & \multicolumn{2}{c}{\textbf{KGPFN (fine-tuned)}} \\
\textbf{Dataset} & \textbf{MRR} & \textbf{Hits@10} & \textbf{MRR} & \textbf{Hits@10} \\
\midrule
FB15k237\_10    & 0.271$\pm$0.008 & 0.434$\pm$0.002 & 0.288$\pm$0.010 & 0.457$\pm$0.011 \\
FB15k237\_20    & 0.294$\pm$0.029 & 0.444$\pm$0.039 & 0.296$\pm$0.033 & 0.445$\pm$0.033 \\
FB15k237\_50    & 0.321$\pm$0.010 & 0.526$\pm$0.020 & 0.322$\pm$0.012 & 0.532$\pm$0.015 \\
CoDExSmall      & 0.475$\pm$0.021 & 0.685$\pm$0.014 & 0.497$\pm$0.013 & 0.694$\pm$0.013 \\
CoDExLarge      & 0.346$\pm$0.006 & 0.479$\pm$0.008 & 0.349$\pm$0.004 & 0.482$\pm$0.010 \\
NELL995         & 0.516$\pm$0.007 & 0.666$\pm$0.009 & 0.527$\pm$0.010 & 0.670$\pm$0.007 \\
ConceptNet100k  & 0.116$\pm$0.008 & 0.203$\pm$0.011 & 0.291$\pm$0.003 & 0.527$\pm$0.004 \\
DBpedia100k     & 0.418$\pm$0.009 & 0.607$\pm$0.004 & 0.462$\pm$0.012 & 0.623$\pm$0.008 \\
YAGO310         & 0.432$\pm$0.010 & 0.654$\pm$0.014 & 0.557$\pm$0.008 & 0.717$\pm$0.016 \\
AristoV4        & 0.205$\pm$0.027 & 0.327$\pm$0.037 & 0.342$\pm$0.015 & 0.494$\pm$0.023 \\
Hetionet        & 0.276$\pm$0.011 & 0.438$\pm$0.018 & 0.405$\pm$0.014 & 0.546$\pm$0.012 \\
WDsinger        & 0.450$\pm$0.010 & 0.585$\pm$0.017 & 0.454$\pm$0.016 & 0.589$\pm$0.017 \\
NELL23k         & 0.317$\pm$0.019 & 0.555$\pm$0.011 & 0.333$\pm$0.019 & 0.559$\pm$0.012 \\
\bottomrule
\end{tabular}
\end{table}

\begin{table}[H]
\centering
\caption{Results on inductive datasets (MRR / Hits@10).}
\label{tab:inductive}
\small
\begin{tabular}{lcc|cc}
\toprule
 & \multicolumn{2}{c|}{\textbf{KGPFN (w/o FT)}} & \multicolumn{2}{c}{\textbf{KGPFN (fine-tuned)}} \\
\textbf{Dataset} & \textbf{MRR} & \textbf{Hits@10} & \textbf{MRR} & \textbf{Hits@10} \\
\midrule
FB15k237v1   & 0.524$\pm$0.002 & 0.686$\pm$0.003 & 0.531$\pm$0.001 & 0.703$\pm$0.006 \\
FB15k237v2   & 0.509$\pm$0.009 & 0.723$\pm$0.008 & 0.520$\pm$0.010 & 0.736$\pm$0.006 \\
FB15k237v3   & 0.498$\pm$0.008 & 0.663$\pm$0.004 & 0.508$\pm$0.005 & 0.672$\pm$0.006 \\
FB15k237v4   & 0.493$\pm$0.009 & 0.693$\pm$0.007 & 0.495$\pm$0.006 & 0.698$\pm$0.010 \\
WN18RRv1     & 0.709$\pm$0.003 & 0.812$\pm$0.007 & 0.725$\pm$0.000 & 0.826$\pm$0.002 \\
WN18RRv2     & 0.707$\pm$0.008 & 0.809$\pm$0.006 & 0.722$\pm$0.008 & 0.816$\pm$0.008 \\
WN18RRv3     & 0.397$\pm$0.004 & 0.512$\pm$0.007 & 0.544$\pm$0.005 & 0.647$\pm$0.004 \\
WN18RRv4     & 0.663$\pm$0.004 & 0.738$\pm$0.009 & 0.674$\pm$0.006 & 0.744$\pm$0.004 \\
NELLv1       & 0.783$\pm$0.005 & 0.886$\pm$0.005 & 0.798$\pm$0.002 & 1.000$\pm$0.001 \\
NELLv2       & 0.578$\pm$0.010 & 0.775$\pm$0.010 & 0.604$\pm$0.009 & 0.777$\pm$0.006 \\
NELLv3       & 0.611$\pm$0.014 & 0.787$\pm$0.012 & 0.612$\pm$0.015 & 0.797$\pm$0.013 \\
NELLv4       & 0.567$\pm$0.016 & 0.811$\pm$0.020 & 0.585$\pm$0.009 & 0.811$\pm$0.009 \\
ILPC-small   & 0.306$\pm$0.006 & 0.461$\pm$0.008 & 0.311$\pm$0.009 & 0.467$\pm$0.013 \\
ILPC-large   & 0.309$\pm$0.007 & 0.439$\pm$0.013 & 0.321$\pm$0.005 & 0.443$\pm$0.008 \\
HM-1k        & 0.076$\pm$0.001 & 0.119$\pm$0.002 & 0.084$\pm$0.000 & 0.146$\pm$0.001 \\
HM-3k        & 0.059$\pm$0.002 & 0.096$\pm$0.012 & 0.079$\pm$0.004 & 0.101$\pm$0.007 \\
HM-5k        & 0.051$\pm$0.006 & 0.082$\pm$0.002 & 0.105$\pm$0.003 & 0.146$\pm$0.002 \\
HM-indigo    & 0.443$\pm$0.011 & 0.662$\pm$0.014 & 0.464$\pm$0.006 & 0.684$\pm$0.005 \\
\bottomrule
\end{tabular}
\end{table}

\begin{table}[H]
\centering
\caption{Results on fully inductive datasets (MRR / Hits@10).}
\label{tab:full_inductive}
\small
\begin{tabular}{lcc|cc}
\toprule
 & \multicolumn{2}{c|}{\textbf{KGPFN (w/o FT)}} & \multicolumn{2}{c}{\textbf{KGPFN (fine-tuned)}} \\
\textbf{Dataset} & \textbf{MRR} & \textbf{Hits@10} & \textbf{MRR} & \textbf{Hits@10} \\
\midrule
NL-0         & 0.365$\pm$0.008 & 0.606$\pm$0.008 & 0.409$\pm$0.009 & 0.615$\pm$0.013 \\
NL-25        & 0.401$\pm$0.004 & 0.607$\pm$0.009 & 0.414$\pm$0.009 & 0.618$\pm$0.005 \\
NL-50        & 0.409$\pm$0.009 & 0.574$\pm$0.008 & 0.422$\pm$0.009 & 0.618$\pm$0.004 \\
NL-75        & 0.385$\pm$0.003 & 0.570$\pm$0.005 & 0.389$\pm$0.003 & 0.576$\pm$0.007 \\
NL-100       & 0.494$\pm$0.007 & 0.695$\pm$0.011 & 0.496$\pm$0.004 & 0.730$\pm$0.004 \\
FB-25        & 0.401$\pm$0.018 & 0.654$\pm$0.035 & 0.403$\pm$0.012 & 0.662$\pm$0.029 \\
FB-50        & 0.347$\pm$0.012 & 0.553$\pm$0.020 & 0.359$\pm$0.013 & 0.578$\pm$0.017 \\
FB-75        & 0.414$\pm$0.009 & 0.619$\pm$0.003 & 0.419$\pm$0.007 & 0.629$\pm$0.005 \\
FB-100       & 0.442$\pm$0.007 & 0.655$\pm$0.011 & 0.444$\pm$0.007 & 0.656$\pm$0.009 \\
WK-25        & 0.324$\pm$0.002 & 0.535$\pm$0.007 & 0.327$\pm$0.003 & 0.537$\pm$0.012 \\
WK-50        & 0.168$\pm$0.019 & 0.330$\pm$0.006 & 0.186$\pm$0.019 & 0.343$\pm$0.008 \\
WK-75        & 0.384$\pm$0.012 & 0.542$\pm$0.012 & 0.386$\pm$0.011 & 0.559$\pm$0.007 \\
WK-100       & 0.182$\pm$0.015 & 0.299$\pm$0.022 & 0.193$\pm$0.014 & 0.324$\pm$0.016 \\
MT1-health   & 0.367$\pm$0.008 & 0.477$\pm$0.012 & 0.386$\pm$0.009 & 0.478$\pm$0.010 \\
MT1-tax      & 0.287$\pm$0.008 & 0.348$\pm$0.002 & 0.428$\pm$0.005 & 0.576$\pm$0.003 \\
MT2-org      & 0.103$\pm$0.003 & 0.181$\pm$0.003 & 0.104$\pm$0.002 & 0.186$\pm$0.005 \\
MT2-sci      & 0.305$\pm$0.003 & 0.438$\pm$0.004 & 0.374$\pm$0.009 & 0.548$\pm$0.006 \\
MT3-art      & 0.274$\pm$0.007 & 0.454$\pm$0.020 & 0.334$\pm$0.011 & 0.512$\pm$0.016 \\
MT3-infra    & 0.647$\pm$0.009 & 0.813$\pm$0.012 & 0.661$\pm$0.007 & 0.823$\pm$0.011 \\
MT4-sci      & 0.307$\pm$0.006 & 0.489$\pm$0.005 & 0.324$\pm$0.009 & 0.512$\pm$0.011 \\
MT4-health   & 0.618$\pm$0.001 & 0.768$\pm$0.001 & 0.718$\pm$0.002 & 0.805$\pm$0.003 \\
Metafam      & 0.459$\pm$0.001 & 0.932$\pm$0.000 & 1.000$\pm$0.001 & 1.000$\pm$0.000 \\
FBNELL       & 0.493$\pm$0.007 & 0.686$\pm$0.007 & 0.513$\pm$0.004 & 0.689$\pm$0.001 \\
\bottomrule
\end{tabular}
\end{table}

\section{AI use statement}

We used generative AI tools to polish the language of the manuscript and to assist with writing, as well as to help implement experimental and analysis code. We also used AI assistance in drafting the proofs of the theoretical results; the authors have independently checked the correctness of all statements and proofs. We have reviewed all AI-assisted text and code, and take full responsibility for the final content of this paper, including its claims, wording, and experimental results.

\end{document}